\documentclass[12pt,a4paper]{article}

\usepackage{fancyhdr}
\usepackage{amsmath,mathtools}
\usepackage{bm}
\usepackage{esint}
\usepackage{amsfonts}
\usepackage{amsthm}
\usepackage{amssymb}
\usepackage{amsmath}
\usepackage{amsbsy}
\usepackage{verbatim}
\usepackage{graphicx}
\usepackage{subcaption}
\usepackage{url}
\usepackage{mathrsfs}
\usepackage[hmargin = 1in,vmargin=.75in]{geometry}
\usepackage{color}
\usepackage{amstext}
\usepackage{stmaryrd}
\usepackage{enumerate}
\usepackage{algpseudocode}
\usepackage{algorithm}
\usepackage{pifont}
\usepackage{prettyref}
\usepackage{subcaption}
\usepackage[dvipsnames]{xcolor}

\usepackage{cite}
\usepackage{graphicx}
\usepackage{multicol}
\usepackage{amsmath}
\usepackage{bbm}
\usepackage{amssymb}
\usepackage[shortlabels]{enumitem}
\usepackage{amsthm}

\font\msbm=msbm10

\numberwithin{equation}{section}

\theoremstyle{plain}
\newtheorem{theorem}{Theorem}[section]
\newtheorem{lemma}[theorem]{Lemma}

\def\mathbb#1{\hbox{\msbm{#1}}}

\newcommand{\lp}{\left(} 
\newcommand{\rp}{\right)} 
\newcommand{\ls}{\left[} 
\newcommand{\rs}{\right]} 
\newcommand{\lc}{\left\{} 
\newcommand{\rc}{\right\}} 
\newcommand{\abs}[1]{ \left|  #1 \right| }

\newcommand{\cond}{\sqrt{\alpha}-\sqrt{\beta}>\sqrt{2}}
\newcommand{\notcond}{\sqrt{\alpha}-\sqrt{\beta}<\sqrt{2}}
\newcommand{\norm}[1]{\left|\left| #1 \right|\right|}
\newcommand{\norminf}[1]{\left|\left| #1 \right|\right|_{\infty}}
\newcommand{\normtwotoinf}[1]{\left|\left| #1 \right|\right|_{2,\infty}}
\newcommand{\normfro}[1]{\left|\left| #1 \right|\right|_{F}}
\newcommand{\E}[1]{\mathbbm{E} #1 }
\newcommand{\prb}[1]{\mathbbm{P} \left( #1\right)  }
\newcommand{\al}{\alpha }
\newcommand{\be}{\beta }

\newcommand{\ones}{\mathbbm{1} }
\newcommand{\Lm}{L^{(m)} }
\newcommand{\Am}{A^{(m)} }
\newcommand{\Dm}{D^{(m)} }
\newcommand{\um}{u^{(m)}_2 }
\newcommand{\mrow}[1]{#1 _{m\cdot} }
\newcommand{\irow}[1]{#1 _{i\cdot} }
\newcommand{\Lsym}{\mathcal{L}}

\newcommand{\inner}[2]{\left\langle #1,#2 \right\rangle }
\newcommand{\bigo}[1]{ O\left( #1\right) }
\newcommand{\smallo}[1]{ o\left( #1\right) }
\newcommand{\Dhalf}{D^{\frac{1}{2}} }
\newcommand{\Dneghalf}{D^{-\frac{1}{2}} }
\newcommand{\Dneg}{D^{-1} }
\newcommand{\uo}{u^{\perp} }
\newcommand{\dm}{d^{(m)} }
\newcommand{\dmin}{d_{\min}}
\newcommand{\bd}{d}
\newcommand{\bdout}{{d}_{\text{out}}}
\newcommand{\sgn}[1]{ \text{sgn}\lp #1 \rp}
\newcommand{\dmax}{d_{\max}}
\newcommand{\Nhalf}{N^{\frac{1}{2}} }
\newcommand{\Nneghalf}{N^{-\frac{1}{2}} }
\newcommand{\Dstarneghalf}{(D^*)^{-\frac{1}{2}} }
\newcommand{\utwostar}{u_2^*}
\newcommand{\Aout}{A_{\text{out}}}
\newcommand\numberthis{\addtocounter{equation}{1}\tag{\theequation}}
\newcommand{\douti}{d_{\text{out}}^{(i)}}
\newcommand{\dini}{d_{\text{in}}^{(i)}}
\newcommand{\Ain}{A_{\text{in}}}
\newcommand{\bdin}{d_{\text{in}}}
\newcommand{\f}{f(\xi;\al,\be)}
\newcommand{\Lsymm}{\mathcal{L}^{(m)}}
\newcommand{\dminm}{d_{\min}^{(m)}}
\newcommand{\dmaxm}{d_{\max}^{(m)}}

\DeclareMathOperator{\diag}{diag}

\begin{document}

\title{\bf Strong Consistency, Graph Laplacians, and the Stochastic Block Model}

\author{Shaofeng Deng\thanks{Department of Mathematics, University of California at Davis (Email: sfdeng@math.ucdavis.edu).}, Shuyang Ling\thanks{New York University Shanghai (Email:  sl3635@nyu.edu).}~~and Thomas Strohmer\thanks{Center of Data Science and Artificial Intelligence Research and Department of Mathematics, University of California at Davis (Email: strohmer@math.ucdavis.edu).}\,\,\thanks{S.D. and T.S.\ acknowledge support from the NSF via grants DMS 1620455 and DMS 1737943.} }

\maketitle
\begin{abstract} 
Spectral clustering has become one of the most popular algorithms in data clustering and community detection. We study the performance of  classical two-step spectral clustering via the graph Laplacian  to learn the stochastic block model. 
Our aim is to answer the following question: when is spectral clustering via the graph Laplacian able to achieve strong consistency, i.e., the exact  recovery of the underlying hidden communities? 
Our work provides an entrywise analysis (an $\ell_{\infty}$-norm perturbation bound) of the Fielder eigenvector of both the unnormalized  and the normalized Laplacian associated with the adjacency matrix sampled from the stochastic block model. 
We prove that spectral clustering is able to achieve exact recovery of the planted community structure under conditions that match the information-theoretic limits. 
\end{abstract}

\section{Introduction}



Data with network structure are ubiquitous, ranging from biological network to social and web networks~\cite{GN02,NM03}. Among many networks, one of the most significant features is community structure or clustering, i.e., a subset of vertices in a huge network are strongly connected while the inter-community connectivity is relatively weak. 
Detecting community structure in networks is one central problem across several scientific fields: how to infer the hidden community structure from the linkage among vertices? A vast amount of research has been done to solve the challenging community detection problem~\cite{F10,LFR08,GN02,NM03}. In particular, community detection with random block structure is an intriguing topic for researchers in mathematics, computer science, physics, and statistics.
One prominent example is the stochastic block model (SBM), which is originally proposed in~\cite{HLL83} to study social networks. Now it has become a benchmark model for comparing different community detection methods. A recent surge of research activities is devoted to designing a variety of algorithms and methods to either detect or recover the hidden community with emphasis on understanding the fundamental limits for community detection in connection with the SBM~\cite{Abbe17}.

On the other hand, spectral clustering is one of most widely used techniques in  data clustering. The classical spectral clustering follows the well-known two-step procedure: Laplacian eigenmap and rounding~\cite{Von07,NJW02,SM00,BN02}.
Despite its popularity and empirical success in numerous applications, its theoretical understanding is still relatively limited. The main difficulty lies in obtaining an entrywise analysis of the Fielder eigenvector of the  graph Laplacian.

In this work, we will study the performance of spectral clustering in community detection for the stochastic block model. We denote by ${\cal G}(n,p_n,q_n)$ the stochastic block model with a total of $n$ vertices and $n/2$ vertices for each community; the adjacency matrix $A = (A_{ij})_{1\leq i,j\leq n}$ of this network is a symmetric matrix which has its $(i,j)$-entry an independent Bernoulli random variable:
\[
\mathbb{P}(a_{ij} = 1) = \begin{cases}
p_n,\quad \text{ if } (i,j)\text{ are in the same community}, \\
q_n,\quad \text{ if } (i,j)\text{ are in different communities},
\end{cases}
\]
where $p_n > q_n$ for all $n$. Note that the parameters $p$ and $q$ usually depend on $n$; for simplicity, we replace $p_n$ and $q_n$ by $p$ and $q$ if there is no confusion.

We focus on answering the following fundamental questions: under what conditions on $(n,p,q)$ is the classical two-step spectral clustering method able to recover the underlying hidden communities exactly?  Moreover, we are interested in the optimality of spectral clustering: does spectral clustering work even if the triple $(n,p,q)$ is close to the information-theoretic limits?

\subsection{Related work and our contributions} 

As all three topics, community detection, spectral clustering, and stochastic block models, have received extensive attention, it is not surprising that there exists a large amount of literature on each of them. While an exhaustive literature review is beyond the scope of this paper, we will briefly review each of these topics, and highlight those contributions that have inspired our research.

Community detection for general networks is well studied and has found many applications. 
We refer interested readers to~\cite{F10,GN02,NM03} for more details on this topic. Spectral clustering~\cite{NJW02,SM00,BN02,Von07}, which is based on the  graph Laplacian~\cite{Chung97}, plays an important role in data- and network-clustering. It is closely related to finding the globally optimal ratio cut and normalized cut of a given graph. In fact, spectral clustering is a natural spectral relaxation of the NP-hard ratio/normalized cut minimization problem. Much excellent research has been done to address how well the solutions to these NP-hard problems are approximated by solutions derived from the spectra of graph Laplacians, which includes (higher-order) Cheeger-type inequalities~\cite{Chung97,LGT14}.
However, one theoretical challenge still remains: for what types of graphs is spectral clustering able to recover the globally optimal graph partitioning and the underlying communities? This is pointed out in~\cite{LR15}, where the authors state {\em ``An important future work would be to extend
	some of the results and techniques [...] to spectral clustering using the graph Laplacian''}.
The main bottleneck is the highly challenging problem of providing an entrywise analysis of the Fiedler eigenvector, the eigenvector associated with the second smallest eigenvalue of the graph Laplacian. 
In fact, this major problem regarding the entrywise analysis of Laplacian eigenvectors is also mentioned in~\cite{AFWZ17} as one future research direction.



\vskip0.2cm

The analysis of the stochastic block model originated from~\cite{HLL83} in the study of social networks. Since then, a vast amount of follow-up research has been conducted to understand how to recover the hidden planted partition with efficient polynomial-time algorithms. In particular, we are interested in the fundamental limits of detection and community recovery in the stochastic block model~\cite{Abbe17}. Here, detection is defined as providing a network clustering which is correlated with the underlying true partition~\cite{MNS18}. 
Generally speaking, the sparser the graph is, the more difficult it is to detect or recover the underlying communities. For the model $\mathcal{G}(n,p,q)$ we call the rate at which $p$ and $q$ tend to 0 their sparsity regimes. 
The detection threshold is usually studied for sparser graphs, in particular in the regime $p =a n^{-1} $ and $q = bn^{-1}$ with $a > b$. The work~\cite{DKMZ11} applied the cavity method, a heuristic from statistical physics, to predict that a detection threshold exists for the community detection problem under stochastic block models. Later on, this detection threshold is confirmed by~\cite{MNS15,MNS18,M14}: the detection of community is possible if and only if $(a-b)^2 > 2(a+b)$.

Another line of work on the stochastic block model focuses on correctly recovering from the adjacency matrix the true label of each 
vertex~\cite{Abbe17,ABH16,GV16,HWX16,Vu18,B18,AminiL18,BC09}, which is only possible in denser regimes.  
We say an algorithm achieves weak consistency (or almost exact recovery) if with probability $1-o(1)$, the proportion of misclassified nodes goes to 0 as $n$ goes to infinity. The weak consistency of spectral method in learning stochastic block model is discussed in~\cite{RCY11,LR15,Yun2014AccurateCD,mossel_consistency_2016}. Strong consistency (or exact recovery) on the other hand requires no misclassified node with probability $1-o(1)$. The concept of strong consistency was introduced and investigated in~\cite{BC09}, which is followed by a series of work including a sharp theoretical threshold~\cite{ABH16,mossel_consistency_2016} in the critical regime $p = \alpha n^{-1}\log n$ and $q = \beta n^{-1}\log n$.
This fundamental threshold states that maximal likelihood estimation (MLE) achieves strong consistency if $\sqrt{\alpha} - \sqrt{\beta} > \sqrt{2}$ and no algorithm can achieve strong consistency if $\sqrt{\alpha} - \sqrt{\beta} < \sqrt{2}$. Among all the existing approaches, semidefinite programming relaxation has proven to be a powerful tool for exact recovery~\cite{GV16,ABH16,HWX16,B18,MS16}. 
In particular, in~\cite{HWX16,B18} it has been shown that SDP relaxation will  find the underlying hidden partition exactly if $\sqrt{\alpha} - \sqrt{\beta} > \sqrt{2}$ with high probability, which is optimal in terms of the information-theoretic limit~\cite{ABH16,mossel_consistency_2016}.

The success of SDP relaxation always comes with a high price: its expensive computational costs are the main roadblock towards practical application. Instead, spectral methods~\cite{BO87,MF01,CO10,RCY11,ZA19,LR15,mossel_consistency_2016,Vu18,Yun2014AccurateCD} are sometimes preferred when tackling large-scale community detection problems. Some spectral methods perform the clustering tasks via the eigenvector of the adjacency matrix or the Laplacian: if the adjacency matrix (Laplacian) is close to its expectation whose eigenvector reveals the hidden partition~\cite{FO05}, then the eigenvector of the adjacency matrix (Laplacian) contains important information which can be used to infer the hidden partition. 
With the help of classical $\ell_2$-norm eigenvector perturbation, mainly based on the Davis-Kahan theorem~\cite{DK70}, one can prove the correct recovery of the majority of the labels by simply taking the sign of the eigenvectors. However, matrix perturbation under $\ell_2$-norm, in spite of its convenience, becomes rather limited in studying the exact recovery of hidden community structure. 
The Davis-Kahan theorem does not give a satisfactory bound of how many labels are correctly classified because $\ell_2$-norm perturbation analysis does not yield a sufficiently tight bound on each entry of the eigenvector. 

As a result, we prefer an $\ell_{\infty}$-norm perturbation bound of eigenvectors when we are concerned with exact recovery. However, it is much more challenging to get an $\ell_{\infty}$-norm perturbation bound for eigenvectors of general matrices.  
Fortunately, recent years have witnessed a series of excellent contributions on  the entrywise analysis of eigenvectors for a family of random matrices~\cite{FWZ18,EBW18,AFWZ17,su2019strong}. Our approach is mainly inspired by the work of Abbe and his co-authors (see~\cite{AFWZ17}), who give an entrywise analysis of eigenvectors with interesting applications in $\mathbb{Z}_2$-synchronization, community detection, and matrix completion. In particular, one application of their work shows that the second eigenvector of the adjacency matrix is strongly consistent down to theoretical limit. The major technical breakthrough is the so-called {\em leave-one-out trick}. One can also find applications of this trick in other examples including synchronization~\cite{ZB18} and the analysis of nonconvex optimization algorithms in signal processing~\cite{MWYC18}. 

It is well worth noting that the result in~\cite{AFWZ17} mainly focuses on studying the eigenvectors of the adjacency matrix which enjoys row/column-wise independence.
However, in our case, the graph Laplacian no longer has this independence. Thus, new techniques need to be developed to overcome this challenge. In~\cite{su2019strong}, the authors study a graph Laplacian based method and prove its strong consistency in the critical regime $p = \alpha n^{-1}\log n$ and $q = \beta n^{-1}\log n$. But they do not show strong consistency for all constants down to theoretical limit $\sqrt{\alpha} - \sqrt{\beta} > \sqrt{2}$.

In this work, we establish an $\ell_{\infty}$-norm perturbation bound for the Fiedler eigenvector of both the unnormalized Laplacian and the normalized Laplacian associated with the stochastic block model. We prove that spectral clustering is able to achieve strong consistency when the triple $(n,p,q)$ satisfies the information-theoretic limits $\sqrt{\alpha} - \sqrt{\beta} > \sqrt{2}$ in~\cite{ABH16,mossel_consistency_2016} where $p = \alpha n^{-1}\log n$ and $q = \beta n^{-1}\log n$. In particular, our analysis of the normalized Laplacian is new and should be of independent interest.

\vskip0.5cm


\subsection{Organization of our paper}
Our paper is organized as follows. Section~\ref{s:prelim} reviews the basics of graph Laplacians, spectral clustering, as well as perturbation theory. We will present the main results, including the strong consistency of spectral clustering, in Section~\ref{s:main}. Numerical experiments are given in Section~\ref{s:numerics} which complement our theoretical analysis. The proofs are delegated to Section~\ref{s:proof}.

\subsection{Notation}
We introduce some notation which will be used throughout this paper.
For any vector $x\in\mathbb{C}^n$, we define $\|x\|_{\infty} = \max_{i}|x_i|$ and $\|x\| = \sqrt{\sum_{i=1}^n x_i^2}$.
For any matrix $M\in\mathbb{C}^{n\times m}$, we denote its conjugate transpose by $M^H$ and its Moore-Penrose inverse by $M^+$. Let $\irow{M}$ be the $i$th row of $M$, which is a row vector.
Let $\|M\|=\max_{\norm{x}=1}\norm{Mx}$ denote the spectral norm, $\|M\|_F : = \sqrt{\sum_{i,j}\left| M_{ij}\right| ^2}$ denote the Frobenius norm and $\normtwotoinf{M}=\max_{\norm{x}=1}\norminf{Mx}=\max_{i}\norm{\irow{M}}$ denote the two-to-infinity norm.
We denote by $\ones_n$ the $n\times 1$ vector with all entries being 1 and let $J_n = \ones_n\ones_n^{\top}$ be the $n\times n$ matrix of all ones. Furthermore, the vector $\sgn{x}$ denotes the entrywise sign of the vector $x$ and $\diag(x)$ denotes a diagonal matrix whose diagonal entries are the same as the vector $x$. Let $f(n)$ and $g(n)$ be two functions. We say $f(n) = O(g(n))$ if $|f(n)|\leq C|g(n)|$ for some positive constant $C$ and $f(n) = o(g(n))$ if $\lim_{n\rightarrow\infty} |f(n)|/|g(n)| = 0$. Moreover, $f(n) = \Omega(g(n))$ if $g(n) = O(f(n))$, $f(n) = \omega(g(n))$ if $g(n) = o(f(n))$, $f(n) = \Theta(g(n))$ if $g(n) = O(f(n))$ and $f(n) = O(g(n))$.

\section{Preliminaries}\label{s:prelim}

\subsection{The Laplacian and spectral clustering}
In this section, we briefly review the basics of spectral clustering which will be frequently used in the discussion later. 
Let $A\in\mathbb{R}^{n\times n}$ be the adjacency matrix where $A_{ij}=1$ if node $i$ and node $j$ are connected and $A_{ij}=0$ if node $i$ and node $j$ are not connected. Let $D=\text{diag}(A\ones_n)$ be the diagonal matrix where $D_{ii}$ is the degree $d_i$ of node $i$, i.e., $d_i = \sum_{j=1}^n A_{ij}$. The unnormalized and normalized Laplacians are defined as 
\[
L:=D-A, \qquad \Lsym : =\Dneghalf L\Dneghalf
\]
respectively. It is a well-known result~\cite{Chung97} that both $L$ and $\Lsym$ are positive semidefinite. Moreover, their smallest eigenvalue is 0 and the corresponding eigenvectors are $\ones_n$ and $\Dhalf\ones_n$, respectively. 

We say $(\lambda,u)$ is an eigenpair of the generalized eigenvalue problem $(M,N)$ if
$$Mu=\lambda Nu.$$
If $N=I$ is the identity then we say $(\lambda,u)$ is an eigenpair of $M$. All eigenvectors are normalized to have unit length if not specifically specified. The unnormalized spectral clustering involves solving the eigenvalue problem $(L,I)$ and the normalized spectral clustering takes many forms due to the following fact.
\begin{align*}
(\lambda,\Dhalf u)\text{ is an eigenpair of }(\Lsym,I)\iff&(\lambda,u)\text{ is an eigenpair of }(L,D)\\
\iff&(\lambda,u)\text{ is an eigenpair of }(\Dneg L,I)\\
\iff&(1-\lambda,u)\text{ is an eigenpair of }(A,D)\\
\iff&(1-\lambda,u)\text{ is an eigenpair of }(\Dneg A,I).
\end{align*}

We order the eigenvalues of $(L,I)$, $(\Lsym,I)$, $(L,D)$, $(\Dneg L,I)$ in increasing order and those of $(A,D)$, $(\Dneg A,I)$ in decreasing order to keep them in correspondence.
\vskip0.25cm

Spectral clustering consists of two steps: (i) compute the Fiedler eigenvector $u$ (here, with a slight abuse of terminology, we call both the eigenvectors with respect to the second smallest eigenvalue of the unnormalized Laplacian $L = D - A$ and of the random walk normalized Laplacian 
$I - D^{-1}A$ the Fiedler eigenvector);
(ii) apply rounding techniques to $u$ to obtain the clusters. In particular, in this paper we simply assign the membership of node $i$ by taking the sign of $u_i$. The spectral clustering algorithm is illustrated for the unnormalized Laplacian and the normalized Laplacian in Algorithm~\ref{alg:L} and Algorithm~\ref{alg:NL}, respectively, see also~\cite{Von07,SM00}.

\begin{algorithm}[h!]
	\caption{Unnormalized spectral clustering}\label{unnormalized}
	\begin{algorithmic}[1]
		\State {\bf Input:} Adjacency matrix $A$.
		\State Compute the unnormalized graph Laplacian $L = D- A$.
		\State Find the eigenvector $u$ of $(L,I)$ that corresponds to the second smallest eigenvalue.
		
		\State Obtain the partitioning based on sgn($u$).
	\end{algorithmic}
	\label{alg:L}
\end{algorithm}
\begin{algorithm}[h!]
	\caption{Normalized spectral clustering}\label{normalized}
	\begin{algorithmic}[1]
		\State {\bf Input:} Adjacency matrix $A$.
		\State Compute the unnormalized graph Laplacian $L = D- A$.
		\State Find the eigenvector $u$ of $(L,D)$ that corresponds to the second smallest eigenvalue. 
		
		\State Obtain the partitioning based on sgn($u$).
	\end{algorithmic}
	\label{alg:NL}
\end{algorithm}

\subsection{Perturbation theory}
Suppose $A$ is an adjacency matrix sampled from two-community symmetric stochastic block model $\mathcal{G}(n,p,q)$. Without loss of generality, we assume the first $n/2$ nodes form one community and the second half nodes form the other one. Let $A^*=\E{A}$ be the expectation of $A$, and then we have
\[
A^* = \begin{pmatrix}
pJ_{n/2} & q J_{n/2} \\
qJ_{n/2} & pJ_{n/2}
\end{pmatrix}
\]
where $p > q.$
Let
\[
D^* := \frac{n(p+q)}{2}I_n, \quad L^* := D^* - A^*, \quad {\cal L}^* := I_n - \frac{2}{n(p+q)}A^*,
\]
which correspond to the degree matrix, unnormalized Laplacian, and normalized Laplacian associated with $A^*.$ 
Then 
$$u_2^*=\frac{1}{\sqrt{n}}\begin{pmatrix}
\ones_{n/2}\\
-\ones_{n/2}
\end{pmatrix}$$
is the eigenvector that corresponds to the second smallest eigenvalue of both $L^*$ and $(L^*, D^*)$. Now one can easily see that running spectral clustering based on $A^*$ gives the perfect result since sgn$(u_2^*)$ exactly recovers the underlying partition. Seeing $A$ as perturbed $A^*$, we study how the eigenvalues and eigenvectors of $L$ (or $\Lsym$) differ from those of $L^*$ (or $\Lsym^*$). For eigenvalue perturbation, we resort to the well-known min-max principle, which gives rise to the famous Weyl's inequality.
\begin{theorem}[{\bf Courant-Fischer-Weyl min-max/max-min principles}]\label{thm:min-max}
	Let $A$ be an $n\times n$ Hermitian matrix with eigenvalues $\lambda_1\leq\cdots\leq\lambda_t\leq\cdots\leq\lambda_n$. For any $d\in\lc1,2,\cdots,n\rc$, write $\mathcal{V}_d$ for the $d$-dimensional subspace of $\mathbb{C}^n$. Then
	$$\lambda_t=\min_{V\in\mathcal{V}_t}\max_{x\in V\backslash\{0\}}\frac{\inner{x}{Ax}}{\inner{x}{x}}=\max_{V\in\mathcal{V}_{n-t+1}}\min_{x\in V\backslash\{0\}}\frac{\inner{x}{Ax}}{\inner{x}{x}}.$$
\end{theorem}
\begin{theorem}[{\bf Weyl}]\label{thm:weyl}
	Let $A$ be an $n\times n$ Hermitian matrix with eigenvalues $\lambda_1\leq\cdots\leq\lambda_n$. Let $B$ be an $n\times n$ Hermitian matrix with eigenvalues $\mu_1\leq\cdots\leq\mu_n$. Suppose the eigenvalues of $A+B$ are $\rho_1\leq\cdots\leq\rho_n$. Then for $i\in\lc1,2,\cdots,n\rc$,
	$$\lambda_i+\mu_1\leq\rho_i\leq\lambda_i+\mu_n.$$
\end{theorem}
For eigenvector perturbation, the Davis-Kahan theorem plays a powerful role in our analysis. Here we state a version of it that allows us to deal with generalized eigenvalue problems, which is particularly useful in the case of normalized spectral clustering. The following theorem essentially follows from the results in~\cite{eisenstat_relative_1998}, but we will give a self-contained proof in Section 5.
\begin{theorem}[{\bf Generalized Davis-Kahan theorem}]\label{thm:DK}
	Consider the generalized eigenvalue problem $Mu=\lambda Nu$ where $M$ is Hermitian and $N$ is Hermitian positive definite. It has the same eigenpairs as the problem $N^{-1}Mu=\lambda u$. Let $X$ be the matrix that has the eigenvectors of $N^{-1}M$ as columns. Then $N^{-1}M$ is diagonalizable and can be written as \begin{align*}
	N^{-1}M=X\Lambda X^{-1}=X_1\Lambda_1 Y_1^H+X_2\Lambda_2 Y_2^H
	\end{align*}
	where 
	$$X^{-1}=\begin{pmatrix}
	X_1 & X_2
	\end{pmatrix}^{-1}=\begin{pmatrix}
	Y_1^H \\ Y_2^H
	\end{pmatrix},\;\;\;\;\Lambda=\begin{pmatrix}
	\Lambda_1&\\
	&\Lambda_2
	\end{pmatrix}.$$
	
	Suppose $\delta=\min_i|\lp\Lambda_2\rp_{i,i}-\hat{\lambda}|$ is the absolute separation of $\hat{\lambda}$ from $\Lambda_2$, then for any vector $\hat{u}$ we have
	$$||P\hat{u}||\leq\frac{\sqrt{\kappa(N)}||(N^{-1}M-\hat{\lambda}I)\hat{u}||}{\delta}$$
	where 
	$P=(Y_2^+)^HY_2^H=I-(X_1^+)^HX_1^H$ is the orthogonal projection matrix onto the orthogonal complement of the column space of $X_1$, $\kappa(N)=||N|| \cdot ||N^{-1}||$ is the condition number of $N$ and $Y_2^+$ is the Moore-Penrose inverse of $Y_2$.
\end{theorem}

When $N=I$ and $(\hat{\lambda},\hat{u})$ is the eigenpair of a matrix $\hat{M}$, we have
$$\sin\theta\leq\frac{\norm{(M-\hat{M})\hat{u}}}{\delta},$$
where $\theta$ is the canonical angle between $\hat{u}$ and the column space of $X_1$. In this case Theorem~\ref{thm:DK} reduces to Davis and Kahan's $\sin\theta$ theorem~\cite{DK70}.

\section{Main results}\label{s:main}
The main goal of this paper is to show that both the unnormalized and normalized spectral clustering achieve strong consistency for the model $\mathcal{G}(n,p,q)$ when $p = \al\log n/n$, $q=\be\log n/n$ and $\cond$. To this end, we develop an entrywise analysis of the Fiedler eigenvector of the unnormalized and normalized Laplacian. But before we talk about the eigenvectors, it is important to ensure that the eigenvalues are properly ``separated''. Here the separation of eigenvalues means the perturbations of the eigenvalues of $L$ (or $\Lsym$) away from those of $L^*$ (or $\Lsym^*$) are smaller than the eigengaps of $L^*$ (or $\Lsym^*$). This is to ensure the second eigenvector ``comes from'' $\utwostar$ and it is essential when applying the Davis-Kahan theorem. Since the first eigenvalue of $L$ or $\Lsym$ is not perturbed at all, we want the second and the third eigenvalue to be separated. Specifically, we want
$$(\lambda_3^*-\lambda_3)+(\lambda_2-\lambda_2^*)<\lambda_3^*-\lambda_2^*. $$
This is where the behaviors of the unnormalized and normalized Laplacian differ greatly. For the normalized Laplacian, we first present a concentration bound for $\norm{\Lsym-\Lsym^*}$ in Section 3.1 that is tighter than the ones in existing literature. This bound gives $\norm{\Lsym-\Lsym^*}=\bigo{1/\sqrt{\log n}}$ while the eigengap $\lambda_3(\Lsym^*)-\lambda_2(\Lsym^*)=\Theta(1)$. Therefore Weyl's theorem automatically ensures the separation of $\lambda_2(\Lsym)$ and $\lambda_3(\Lsym)$. For the unnormalized Laplacian $L$, we have $L-L^*=(D-D^*)-(A-A^*)$. By Lemma~\ref{lem:A-concen} we can bound $\norm{A-A^*}=O(\sqrt{\log n})$. Moreover one can use the Chernoff bound to show that $\norm{D-D^*}=O(\log n)$. Thus $\norm{L-L^*}=O(\log n)$. Noting that the eigengap $\lambda_3(L^*)-\lambda_2(L^*)=\Theta(\log n)$, one can not draw an immediate conclusion that $\lambda_2(L)$ and $\lambda_3(L)$ are separated. We will discuss how to resolve this difficulty in Section 3.2, where we bound the eigenvalues of $L$ and $\Lsym$ in a more general setting. In short, we are able to find $\lambda_2(L)\leq\beta\log n+\bigo{\log n/\sqrt{n}}$ and $\lambda_3(L)\geq (\beta+\epsilon)\log n$ for some $\epsilon>0$, which shows that the eigenvalues are indeed separated.

Finally we give entrywise bounds for the second eigenvector of $L$ and $(L,D)$. Our analysis is mostly inspired by the work of~\cite{AFWZ17} as well as the leave-one-out technique in ~\cite{ZB18,MWYC18}. The core is to find an appropriate approximation to the second eigenvector of $L$ or $(L,D)$. Denote by $\tilde{u}_2$ the choice of approximation and $u_2$ the output eigenvector of the algorithm. An admissible candidate of $\tilde{u}_2$ should satisfy the following two properties:
\begin{enumerate}[(i)]
	\item The entrywise error between $u_2$ and $\tilde{u}_2$ is negligible.
	\item The entries of $\tilde{u}_2$ exactly recover the planted communities and are sufficiently bounded away from zero.
\end{enumerate}

We choose the following particular choices of $\tilde{u}_2$ for the unnormalized and the normalized spectral clustering. 
\begin{itemize}
	\item For the unnormalized spectral clustering, we let
	\begin{equation*}
	\tilde{u}_2=(D-\lambda_2(L)I)^{-1}A\utwostar.
	\end{equation*}
	\item For the normalized spectral clustering, we let
	\begin{equation*}
	\tilde{u}_2=(1-\lambda_2(\Lsym))^{-1}\Dneg A\utwostar. 
	\end{equation*}
	
\end{itemize}
While more detailed discussion will be provided in Section 3.3 on how to prove the two properties of $\tilde{u}_2$, we first present a numerical illustration in Figure~\ref{main}, which implies that these two choices are indeed satisfactory.

\begin{figure}
	
	\begin{subfigure}[h]{0.5\linewidth}
		\includegraphics[width=\linewidth]{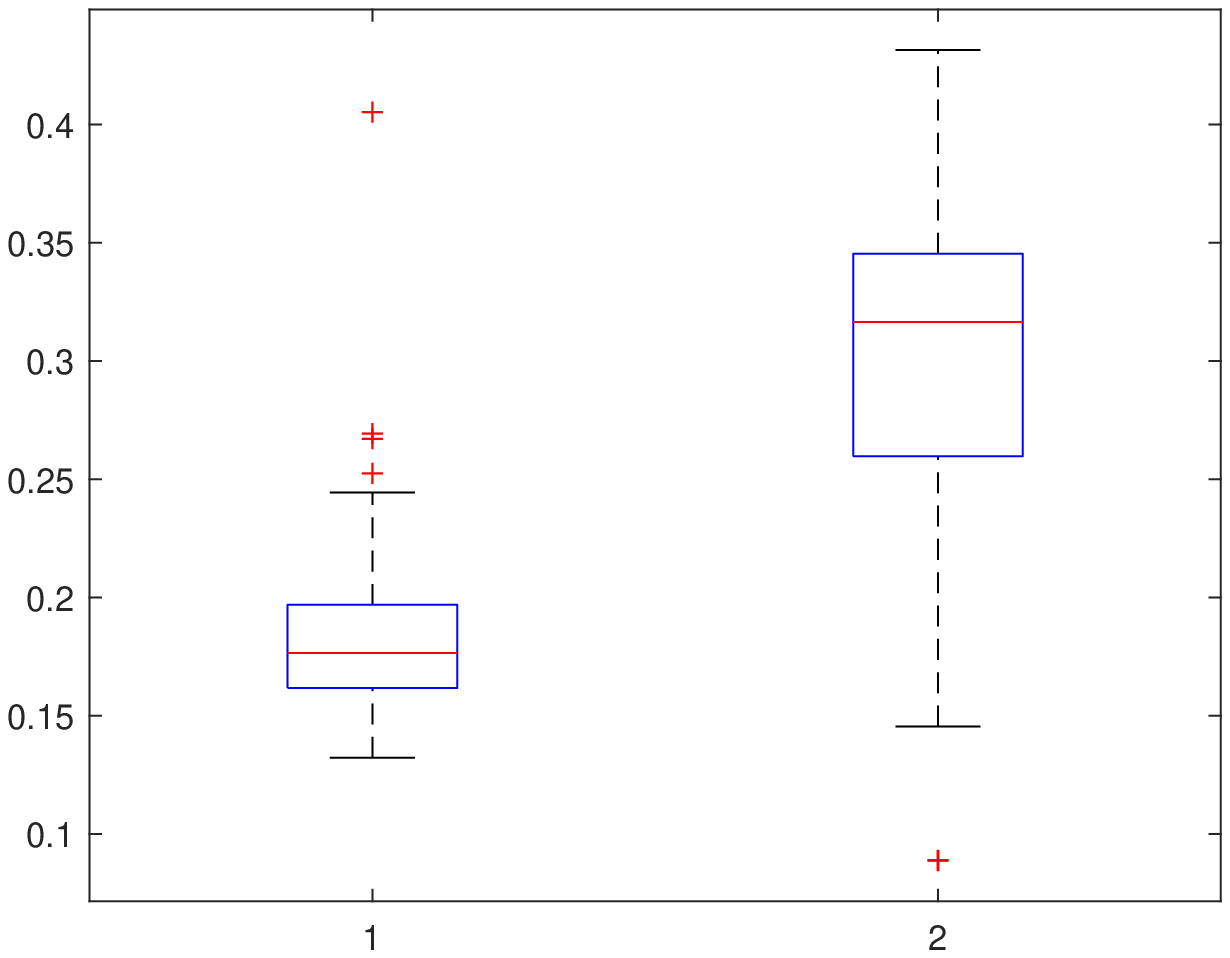}
		\caption{Unnormalized spectral clustering}
	\end{subfigure}
	\hfill
	\begin{subfigure}[h]{0.5\linewidth}
		\includegraphics[width=\linewidth]{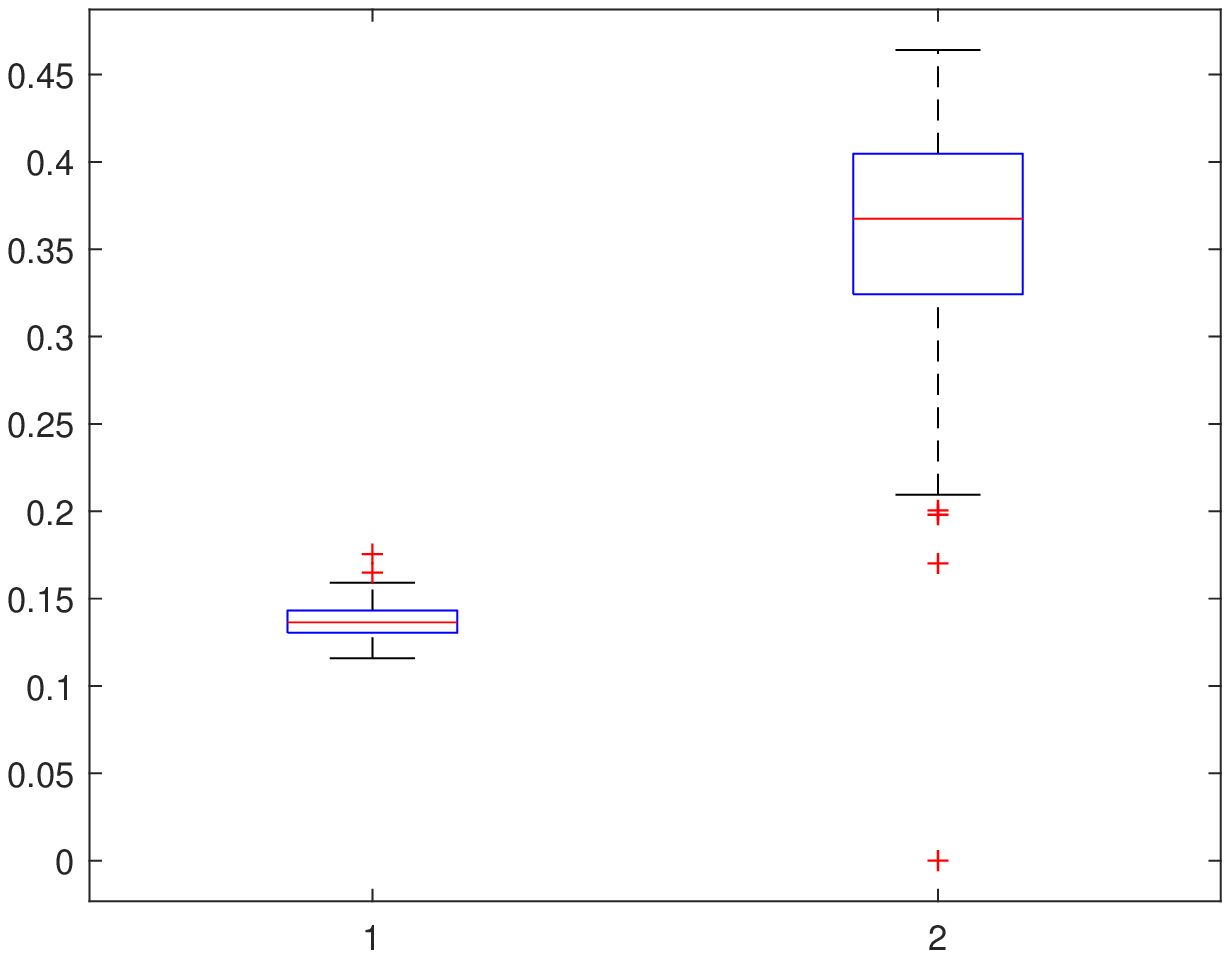}
		\caption{Normalized spectral clustering}
	\end{subfigure}%
	
	\caption{Boxplots  showing the two properties of the approximation $\tilde{u}_2$. For unnormalized spectral clustering $\tilde{u}_2=(D-\lambda_2(L)I)^{-1}A\utwostar$ and for the normalized spectral clustering $\tilde{u}_2=(1-\lambda_2(\Lsym))^{-1}\Dneg A\utwostar$. We fix $n=5000$, $\al=10$, $\be=2$ and the number of trails to be 100. Two quantities (up to sign of $u_2$) are shown in the boxplots: (1) $\sqrt{n}\norminf{u_2-\tilde{u}_2}$; (2) $\sqrt{n}\min\lc z_i(\tilde{u}_2)_i\rc_{i=1}^n$ where $z_i=1$ for $i\leq n/2$ and $z_i=-1$ for $i\geq n/2+1$.}
	\label{main}
\end{figure}

\subsection{Concentration of the normalized Laplacian}
In this section we assume $A$ is an instance of the inhomogeneous Erd\H{o}s-R\'enyi graph on $n$ nodes where node $i$ and $j$ are linked with probability $p_{ij}$. We have the following concentration result for the normalized Laplacian.

\begin{theorem}\label{thm:Lsym-concen}
	Let $A$ be the adjacency matrix of a random graph on $n$ nodes whose edges are sampled independently. Let $A^*=\E{A}=(p_{ij})_{i,j=1,2,\cdots,n}$. Let $\Lsym$ and $\Lsym^*$ be the normalized Laplacian of $A$ and $A^*$ respectively. Assume that $n\max_{ij}p_{ij}\geq c_0\log n$ for some $c_0\geq 1$. Then for any $r>0$, there exists $C=C(c_0,r)$ such that
	\[
	\norm{\Lsym-\Lsym^*}\leq\frac{C\lp n\max_{ij}p_{ij}\rp^{5/2}}{\min\{d_{\min},d^*_{\min}\}^3}
	\] 
	with probability at least $1-n^{-r}$. Here $d_{\min}$ is the minimum degree of $A$ and $d^*_{\min}$ is the minimum degree of $A^*$.
\end{theorem}
Theorem~\ref{thm:Lsym-concen} relies heavily on the following concentration result of the adjacency matrix $A$, which we take directly from  Theorem 5.2 of~\cite{LR15}. 
\begin{lemma}\label{lem:A-concen}
	Let $A$ be the adjacency matrix of a random graph on $n$ nodes whose edges are sampled independently. Let $A^*=\E{A}=(p_{ij})_{i,j=1,2,\cdots,n}$ and assume that $n\max_{ij} p_{ij} \leq d$ for $d \geq c_0 \log n$ and $c_0>0$. Then, for any $r > 0$ there exists a constant $C = C(r, c_0)$ such that
	$$\norm{A-A^*}\leq C\sqrt{d}$$
	with probability at least $1-n^{-r}$.
\end{lemma}
The requirement $n\max_{ij}p_{ij}\geq \log n$ in Theorem~\ref{thm:Lsym-concen} is necessary for concentration. To see this, consider a homogeneous Erd\H{o}s-R\'enyi graph $\mathcal{G}(n,p)$ on $n$ nodes with edges occuring with probability $p$. It is well known that if $np<\log n$ then the graph is asymptotically almost surely disconnected~\cite{Hof16}, causing $\Lsym$ to have multiple 0 eigenvalues, which leads to $\norm{\Lsym-\Lsym^*}\geq 1$. 

The key to applying Theorem~\ref{thm:Lsym-concen} is to control the minimum degree. If $p=\omega(\log n/n)$ in the model $\mathcal{G}(n,p)$, then one can use Chernoff bound to show $\dmin=\Omega(np)$ and thus the concentration reads $\norm{\Lsym-\Lsym^*}=\bigo{1/\sqrt{np}}$. In comparison, the unnormalized Laplacian only has the concentration $\norm{L-L^*}=O(\sqrt{np\log n})$. Indeed, $L-L^*=(D-D^*)-(A-A^*)$ and the Chernoff bound gives $\norm{D-D^*}=O(\sqrt{np\log n})$, Lemma~\ref{lem:A-concen} implies $\norm{A-A^*}=\bigo{\sqrt{np}}$. Noting that $\norm{\Lsym^*}=\Theta(1)$ and $\norm{L^*}=\Theta(np)$, one can see that the concentration of $\Lsym$ is better that of $L$ by a factor $\sqrt{\log n}$. This shows that the concentration of $\Lsym$ is order-wise the same as the concentration of $A$ and better than that of $L$. The bad concentration of $D$ is eliminated by the construction of $\Lsym$.

\subsection{Eigenvalue perturbation}
In this section we assume $A$ is an instance of the block model $\mathcal{G}(n,p,q)$. But we do not assume the sparsity regime of $p$ or $q$.
\subsubsection*{Unnormalized Laplacian}
We have $\lambda_1(L^*)=0$, $\lambda_2(L^*)=nq$, and $\lambda_i(L^*)=n(p+q)/2$ for $i=3,4,\cdots,n$. To keep the second and third eigenvalues of $L$ separated, we want $\norm{L-L^*}$ to be relatively small compared to $\lambda_3(L^*)-\lambda_2(L^*)$, i.e.\ compared to the associated eigengap. Unfortunately this is not always satisfied in the critical regime where $p = \alpha\log n/n$ and $q = \beta\log n/n$ due to the bad concentration of $L$ that we discussed earlier. As we will see, in this regime we have  $\lambda_2(L)\leq\beta\log n+\bigo{n^{-1/2}\log n}$, which means the second eigenvalue is well bounded from above. The challenge is to find a relatively tight lower bound for $\lambda_3(L)$. According to Weyl's theorem and lemma~\ref{lem:A-concen},
\begin{align*}
\lambda_3(L) & \geq \lambda_3(L^*) + \lambda_{\min}(L - L^*) \\
&\geq \lambda_3(L^*)+\lambda_{\min}(D-D^*)-\norm{A-A^*}\\
&=d_{\min}-\bigo{\sqrt{\log n}}.
\end{align*}
Therefore whether the second and the third eigenvalue are separated depends on how well we can bound $\dmin$ from below. Through a Poisson approximation to binomial variables we are able to bound $\dmin$ in the lemma below. 
\begin{lemma}\label{lem:dmin}
	Let $A$ be an instance of $\mathcal{G}(n,p,q)$ where $p=\al\log n/n$ and $\be\log n/n$. Then for any $0<\xi<\frac{\al+\be}{2}$, we have  
	$$\prb{\dmin\geq \frac{\al+\be}{2}\log n-\xi\log n}\geq 1-2n^{-f(\xi;\al,\be)}$$
	for $n$ larger than a constant $N=N(\al,\be)$. Here
	$$f(\xi;\al,\be) = \frac{\al+\be-2\xi}{2}\log\left(\frac{\al+\be-2\xi}{\al+\be} \right)+\xi-1. $$
\end{lemma}
The function $f$ characterizes a trade-off between the perturbation of $\dmin$ and its probability. Note that when $\xi$ is sufficiently close to 0, $f$ will eventually be negative, then Lemma~\ref{lem:dmin} loses its usefulness. To ensure that $\dmin$ is well controlled from below, we introduce the following conditions on the constants $\al$ and $\be$.
\begin{itemize}
	\item[] \textbf{(A1)} There exists $0<\xi<\frac{\al-\be}{2}$ such that $f(\xi;\al,\be)>0$,
	\item[] \textbf{(A2)} $\cond$.
\end{itemize}
From the discussion above, one can see that condition (A1) is enough to ensure $\dmin\geq (\beta+\epsilon)\log n$, which implies the separation of eigenvalues. The condition (A2), which characterizes strong consistency, implies (A1).
\begin{lemma}\label{lem:A2toA1}
	(A2) implies (A1).
\end{lemma}
We define $\bdout\in\mathbb{R}^n$ to be the vector with the $i$th entry being the number of edges between the $i$th node and the community that does not contain the $i$th node. Define $\bdout^*=\E{\bdout}$. The concentration of $\bdout$ around its expectation plays an important role in the perturbation of $\lambda_2(L)$. The eigenvalue perturbation theorem for the unnormalized Laplacian is formally stated below. 
\begin{theorem}\label{thm:lambda(L)}
	Let $A$ be an instance of $\mathcal{G}(n,p,q)$. 
	\begin{enumerate}[(i)]
		
		\item (Lower bound for the third eigenvalue in the critical regime.) Suppose $p = \alpha\log n/n$ and $q = \beta\log n/n$. Then for any $\xi>0$ and $\epsilon>0$ there exists $C=C(\xi,\al,\be,\epsilon)>0$ such that
		$$\lambda_3(L)\geq \frac{\al+\be}{2}\log n-(\xi+\epsilon)\log n$$
		with probability at least $1-Cn^{-f(\xi;\al,\be)}$.
		
		\item (Upper bound for the second eigenvalue.) There holds
		$$\lambda_2(L)\leq  nq+\frac{2}{n}\inner{\bdout-\bdout^*}{\ones_n}.$$

		\item (Lower bound for the second eigenvalue.) For any $p \geq p_0\log n/n$ and $r>0$, there exists $M=M(p_0,r)>0$ such that for $q$ satisfying
		$$\frac{p-q}{\sqrt{p}}\geq M\sqrt{\frac{\log n}{n}},$$
		it holds that
		$$\lambda_2(L)\geq nq+\frac{2}{n}\inner{\bdout-\bdout^*}{\ones_n}+\frac{32||\bdout-\bdout^*||||\bdout||}{n^2(p-q)} $$
		with probability at least $1-3n^{-r}$. 
		
		Moreover, suppose $p = \al\log n/n$ and $q = \be\log n/n$. If $\alpha$ and $\beta$ satisfy (A1) so that there is some constant $0<\xi=\xi(\al,\be)<(\al-\be)/2$ satisfying $f(\xi;\al,\be)>0$, then there exists $C_1,C_2>0$ depending on $\al$, $\be$ and $\xi$ such that
		$$\lambda_2(L)\geq \beta\log n-C_1\sqrt{\log n}$$
		with probability at least $1-C_2n^{-f(\xi;\al,\be)}$.
	\end{enumerate}
\end{theorem}
We leave the terms regarding $\bdout$ in the statement on account of the fact that their behaviors change as the sparsity regime of $q$ changes. Although these terms get smaller as $q$ gets smaller, it is hard to put these relations in a unified form. We provide the following lemma to discuss how to control $||\bdout-\bdout^*||$ and $\inner{\bdout-\bdout^*}{\ones_n}$. The term $\norm{\bdout}$ is then controlled by $||\bdout-\bdout^*||+||\bdout^*||$.
\begin{lemma}\label{lem:dout}
	\begin{enumerate}[(i)]
		\item If $q\geq q_0\log n/n^2$ for some $q_0>0$, then for any $r>0$ there exists $C=C(q_0,r)>0$ such that 
		$$\prb{\left| \inner{\bdout-\bdout^*}{\ones_n}\right| \geq C\sqrt{n^2q\log n}}\leq 2n^{-r}.$$
		\item If $q\geq q_0\log n/n$ for some $q_0>0$, then for any $r>0$ there exists $C=C(q_0,r)>0$ such that 
		$$\prb{\norm{\bdout-\bdout^*} \geq C\sqrt{n^2q}}\leq n^{-r}.$$
		\item If $q\geq q_0/n^2$ for some $q_0>0$, then there exists $C=C(q_0)>0$ such that 
		$$\prb{\norm{\bdout-\bdout^*} \geq C\sqrt{n^2q}}\leq \frac{1}{n}+\frac{0.01q_0}{n^2q}.$$
	\end{enumerate}
\end{lemma}
For $p = \al\log n/n$ and $q = \be\log n/n$ where $\cond$, the eigenvalue perturbation is simply 
$$\beta\log n-O(\sqrt{\log n})\leq \lambda_2(L)\leq \beta\log n+O(\log n/\sqrt{n})$$
and
$$\lambda_3(L)\geq (\beta+\epsilon)\log n$$
for some constant $\epsilon>0$ with probability $1-\bigo{n^{-f(\xi,\al,\be)}}$.

\subsubsection*{Normalized Laplacian}
For $\Lsym^*$, we have $\lambda_1(\Lsym^*)=0$, $\lambda_2(\Lsym^*)=2q/(p+q)$, and $\lambda_i(\Lsym^*)=1$ for $i=3,4,\cdots,n$. We provide a perturbation bound for $\lambda_2(\Lsym)$.
\begin{theorem}\label{thm:lambda(Lsym)}
	Let $A$ be an instance of $\mathcal{G}(n,p,q)$. 
	\begin{enumerate}[(i)]
		\item (Upper bound for the second eigenvalue) Suppose $p\geq p_0/n$ and $q\geq q_0\log n/n^2$ for some $p_0,q_0>0$. Then for any $r>0$ there exists $C_1=C_1(r,p_0)>0$ and $C_2=C_2(r,p_0,q_0)>0$ such that
		$$\prb{\lambda_2(\Lsym)\leq\frac{2q}{p+q}+C_2\frac{\sqrt{q\log n}}{np}}\geq 1-C_1n^{-r}.$$
		
		\item (Lower bound for the second eigenvalue) For any $r>0$ there exists $p_0=p_0(r)>1$ and $M=M(p_0,r)>0$ such that for all $p\geq p_0\log n/n$ and $q\geq q_0\log n/n^2$ satisfying
		$$\frac{p-q}{\sqrt{p}}\geq\frac{M}{\sqrt{n}}$$
		we have
		$$\prb{\lambda_2(\Lsym)\geq \frac{2q}{p+q}-C_1\lp\frac{\sqrt{q\log n}}{np}+\frac{nq+\frac{1}{\sqrt{n}}\norm{\bdout}}{n(p-q)\sqrt{np}}\rp}\geq 1-C_2n^{-r}$$
		for $C_1,C_2>0$ depending on $p_0$ ,$q_0$ and $r$.
		
		Moreover, if $p=\al\log n/n$ and $q=\be\log n/n$ with $\al>2$ then there exists $0<\xi=\xi(\al,\be)<\frac{\al+\be}{2}$ such that $f(\xi;\al,\be)>0$ and
		$$\prb{\lambda_2(\Lsym)\geq \frac{2\be}{\al+\be}-C_3\frac{1}{\sqrt{\log n}}}\geq 1-C_4n^{-f(\xi;\al,\be)}.$$
		for $C_3,C_4>0$ depending on $\al$ ,$\be$ and $\xi$.
	\end{enumerate}
\end{theorem}
As for $\lambda_3(\Lsym)$, one can use Weyl's theorem and the concentration of $\Lsym $ (Theorem~\ref{thm:Lsym-concen}) to give a good bound. For $p = \al\log n/n$ and $q = \be\log n/n$ where $\cond$, the eigenvalue perturbation is simply 
$$\frac{2\be}{\al+\be}-\bigo{\frac{1}{\sqrt{\log n}}}\leq \lambda_2(\Lsym)\leq \frac{2\be}{\al+\be}+\bigo{\frac{1}{\sqrt{n}}}$$
and
$$\lambda_3(\Lsym)\geq 1-\bigo{\frac{1}{\sqrt{\log n}}}.$$
\subsection{Strong consistency}
In this section we assume $A$ is an instance of $\mathcal{G}(n,p,q)$, $p = \alpha\log n/n$, $q = \beta\log n/n$ and $\cond$. 
\subsubsection*{Unnormalized spectral clustering}
The goal of the following discussion is to give a proof sketch of Theorem~\ref{thm:L:strong-consis}.
\begin{theorem}\label{thm:L:strong-consis}
	Let $p=\al\log n/n$, $q=\be\log n/n$ and $\cond$. Then there exists $\eta= \eta(\al,\be) >0$ and $s\in\lc\pm 1\rc$ such that with probability $1-o(1)$, 
	$$\sqrt{n}(su_2)_{i} \geq\eta \text{  for } i\leq\frac{n}{2}$$
	and
	$$\sqrt{n}(su_2)_{i} \leq-\eta \text{  for } i\geq\frac{n}{2}+1.$$
	
\end{theorem}
One can see Theorem~\ref{thm:L:strong-consis} implies that the unnormalized spectral clustering achieves strong consistency down to the information theoretical limits. Let the vector $(D-\lambda_2(L)I)^{-1}Au_2^*$ be the approximation to $u_2$, the second eigenvector of $L$. Theorem~\ref{thm:L:strong-consis} follows after the following two claims. With probability $1-o(1)$,
\begin{enumerate}[(i)]
	\item $\norminf{u_2-(D-\lambda_2I)^{-1}Au_2^*}=o(1/\sqrt{n})$;
	\item $\sgn{(D-\lambda_2(L)I)^{-1}Au_2^*}$ exactly recovers the planted communities and $$\left|\lp(D-\lambda_2(L)I)^{-1}Au_2^*\rp_i\right|\geq\frac{\eta}{\sqrt{n}}$$
	for all $i$ and some $\eta>0$.
\end{enumerate}
The two claims are up to sign of $u_2$, meaning we write $su_2$ ($s\in\lc1,-1\rc$) simply as $u_2$. We first look at claim (ii). Note that $\dmax-\lambda_2(L)=\bigo{\log n}$, it boils down to showing that the entries of $Au^*_2$ are well bounded away from zero by an order of $\log n/\sqrt{n}$. Since each entry of $Au^*_2$ can be expressed as the difference of two independent binomial variables, an inequality that was introduced in~\cite{Abbe17,ABH16} gives the desired tail bound.
\begin{lemma}\label{lem:BinomialDiff}
	Suppose $\al\geq\be$, $\lc W_i\rc_{i=1}^{n/2}$ are $i.i.d$ Bernoulli($\al\log n/n$), and $\lc Z_i\rc_{i=1}^{n/2}$ are $i.i.d$ Bernoulli($\be\log n/n$), independent of $\lc W_i\rc_{i=1}^{n/2}$. For any $\epsilon\in\mathbb{R}$, we have
	$$\prb{\sum_{i=1}^{n/2}W_i-\sum_{i=1}^{n/2}Z_i\leq \epsilon\log n}\leq n^{-(\sqrt{\al}-\sqrt{\be})^2/2+\epsilon\log(\al/\be)/2}.$$
\end{lemma}
To prove claim (i), note $(D-\lambda_2I)u_2=Au_2$ and expand
$$u_2-(D-\lambda_2I)^{-1}Au_2^*=(D-\lambda_2I)^{-1}A(u_2-u_2^*).$$
We have established that $\dmin\geq(\be+\epsilon)\log n$ and $\lambda_2\leq\be\log n+\bigo{\log n/\sqrt{n}}$, therefore $\norminf{(D-\lambda_2I)^{-1}}=\bigo{1/\log n}$. It remains to show that
\[
\norminf{A(u_2-u_2^*)}=\smallo{\frac{\log n}{\sqrt{n}}}.
\]
This quantity is at the center of both unnormalized and normalized spectral clustering. The technique that we use to control $\norminf{A(u_2-u_2^*)}$ is originated from~\cite{AFWZ17}, in which a row-concentration property of $A$ is the key. We cite the row-concentration in the following lemma.
\begin{lemma}[{\bf Row-concentration property of the adjacency matrix}]\label{lem:row-con}
	Let $w\in\mathbb{R}^n$ be a fixed vector, $\lc X_i\rc_{i=1}^{n}$ be independent random variable where $X_i\sim\text{Bernoulli}(p_i)$. Suppose $p\geq\max_i p_i$ and $a>0$. Then
	$$\prb{\left| \sum_{i=1}^nw_i(X_i-\E{X_i})\right|\geq\frac{(2+a)pn}{1\vee\log\lp\frac{\sqrt{n}\norminf{w}}{\norm{w}}\rp}\norminf{w} }\leq 2e^{-a np}.$$
\end{lemma}
The row-concentration property of $A$ is probabilistic, meaning $w$ and $A$ must be independent. But $(u_2-u_2^*)$ and $A$ are not independent. To overcome this, we use the recently developed and popularized leave-one-out technique. Specifically we consider an auxiliary vector $\um$ defined as the second eigenvector of $\Lm$, the unnormalized Laplacian matrix of $\Am$, where $\Am$ is constructed in a way that $\Am=A$ everywhere except for the $m$-th row and $m$-th column which are replaced by those of $A^*$. The purpose of this auxiliary vector is that the $m$-th row of $A$, denoted by $\mrow{A}$, is now independent of $(\um-u^*)$. Thus the $m$-th entry of $A(u_2-u_2^*)$ is bounded by
\[
\abs{\mrow{A}\lp u_2-u_2^*\rp} \leq\abs{\mrow{A}\lp u_2-\um\rp}+\abs{\mrow{A}\lp \um-u_2^*\rp}.
\]
The first term in the right hand side is well bounded by the small $\ell_2$-norm of $\lp u_2-\um\rp$. In fact, by exploiting the structural difference of $L$ and $\Lm$, the Davis-Kahan theorem eventually gives the bound
\[
\norm{u_2-\um}=\bigo{\norminf{u_2}}.
\]
Using this in conjunction with the fact that
\[
\normtwotoinf{A}\leq\normtwotoinf{\E{A}}+\norm{A-\E{A}}=\bigo{\sqrt{\log n}},
\]
we are able to bound the first term
\[
\abs{\mrow{A}\lp u_2-\um\rp}\leq\normtwotoinf{A}\norm{u_2-\um}=\bigo{\sqrt{\log n}\norminf{u_2}}.
\]
For the second term, we can now use the row-concentration property which yields
\[
\abs{\mrow{A}\lp \um-u_2^*\rp}=\bigo{\frac{\log n\norminf{u_2}}{\log\log n}}.
\]
Thus
\[
\norminf{A(u_2-u_2^*)}=\smallo{\log n\norminf{u_2}}.
\]
Finally we prove $\norminf{u_2}=\bigo{1/\sqrt{n}}$. Indeed, 
$$\norminf{u_2}=\norminf{(D-\lambda_2 I)^{-1}Au_2}\leq\norminf{(D-\lambda_2 I)^{-1}Au_2^*}+\norminf{(D-\lambda_2 I)^{-1}A(u_2-u_2^*)}.$$
Noting that $\norminf{(D-\lambda_2 I)^{-1}A(u_2-u_2^*)}=o(\norminf{u_2})$, the second term on the right hand side is thus absorbed into the left hand side. Therefore
$$\norminf{u_2}=\bigo{\norminf{(D-\lambda_2 I)^{-1}Au_2^*}}=\bigo{\frac{1}{\sqrt{n}}}.$$
Claim (i) then follows.

\subsubsection*{Normalized spectral clustering}
The proof for the normalized spectral clustering is similar to its unnormalized counterpart, albeit more technically involved. Let $u_2$ be the eigenvector of $(L,D)$ that corresponds to the second smallest eigenvalue $\lambda_2(\Lsym)$. We use the vector $(1-\lambda_2(\Lsym))^{-1}\Dneg A\utwostar$ as an approximation to $u_2$. Then we prove with probability $1-o(1)$,
\begin{enumerate}[(i)]
	\item $\norminf{u_2-(1-\lambda_2(\Lsym))^{-1}\Dneg Au_2^*}=o(1/\sqrt{n})$;
	\item $\sgn{(1-\lambda_2(\Lsym))^{-1}\Dneg Au_2^*}$ exactly recovers the planted communities and $$\left|\lp(1-\lambda_2(\Lsym))^{-1}\Dneg Au_2^*\rp_i\right|\geq\frac{\eta}{\sqrt{n}}$$
	for all $i$ and some $\eta>0$.
\end{enumerate}

\begin{theorem}\label{thm:Lsym:strong-consis}
	Let $p=\al\log n/n$, $q=\be\log n/n$ and $\cond$. Then there exists $\eta= \eta(\al,\be) >0$ and $s\in\lc\pm 1\rc$ such that with probability $1-o(1)$, 
	$$\sqrt{n}(su_2)_{i} \geq\eta \text{  for } i\leq\frac{n}{2}$$
	and
	$$\sqrt{n}(su_2)_{i} \leq-\eta \text{  for } i\geq\frac{n}{2}+1.$$
	
\end{theorem}

\section{Numerical explorations}\label{s:numerics}
We illustrate the strong consistency of both spectral clustering methods in Figure~\ref{phase}. It can be clearly seen that both methods achieve strong consistency down to the theoretical threshold $\cond$. The major behavioral difference between the two methods is when we are below this threshold, namely when $\al>\be$ but $\notcond$. In this region, strong consistency is impossible but weak consistency is possible. In Figure~\ref{agreement} we plot the empirical average agreement for each method. Here the~\emph{agreement} is defined as the proportion of the correctly classified nodes. We see that the normalized spectral clustering performs much better in the region between the red line and the green line. The unnormalized spectral clustering does not work as well as the normalized counterpart does since the unnormalized Laplacian is unable to preserve the ``order'' of the eigenvalues (in the sense discussed at the beginning of Section~\ref{s:main}). This shows that the bad concentration of $L$ indeed causes trouble in this sparsity regime. In fact we are able to find an eigenvector of $L$ that has a high agreement, but often this eigenvector is not the Fiedler eigenvector.  
\begin{figure}[h!]
	
	\begin{subfigure}[h]{0.5\linewidth}
		\includegraphics[width=\linewidth]{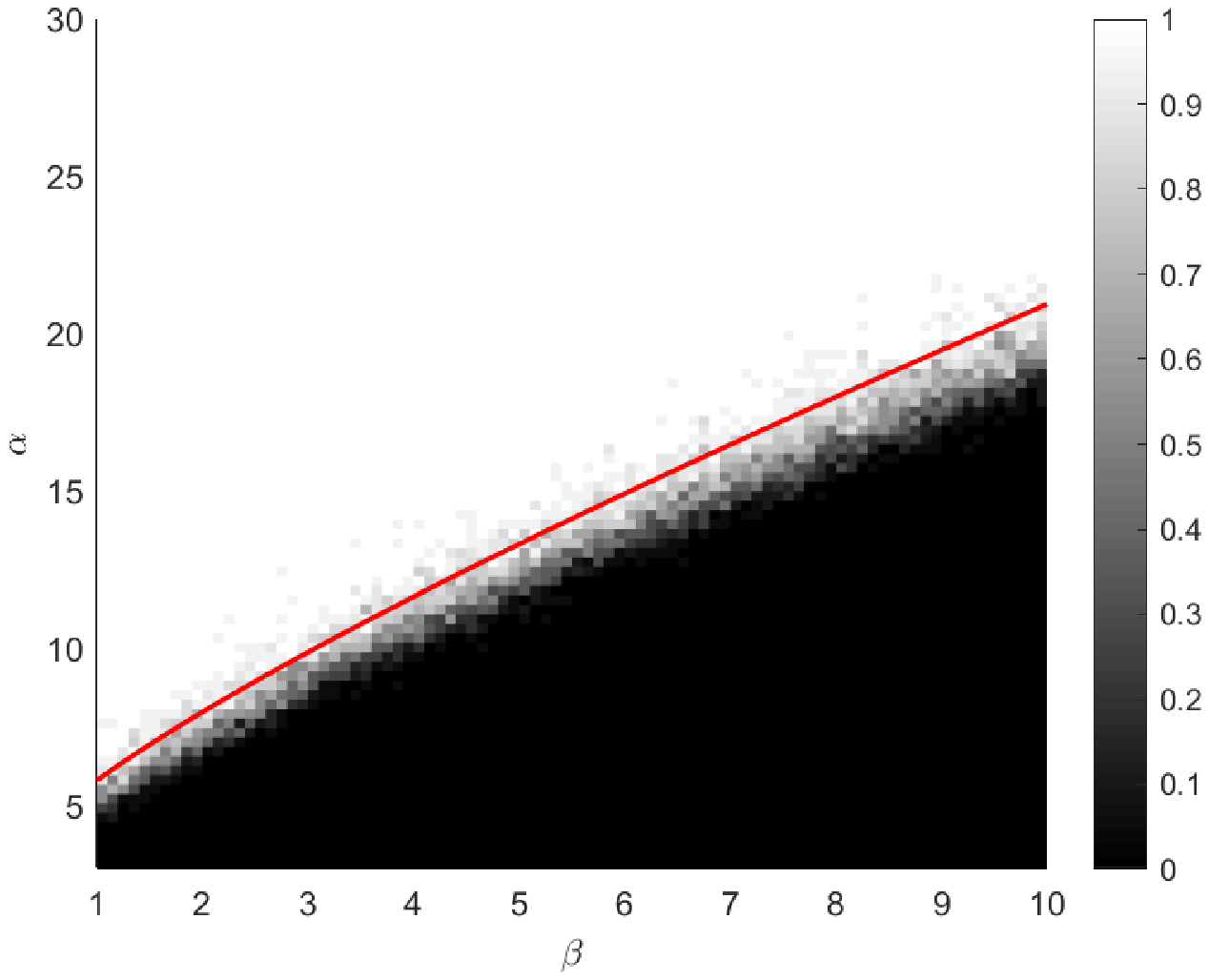}
		\caption{Unnormalized spectral clustering}
	\end{subfigure}
	\hfill
	\begin{subfigure}[h]{0.5\linewidth}
		\includegraphics[width=\linewidth]{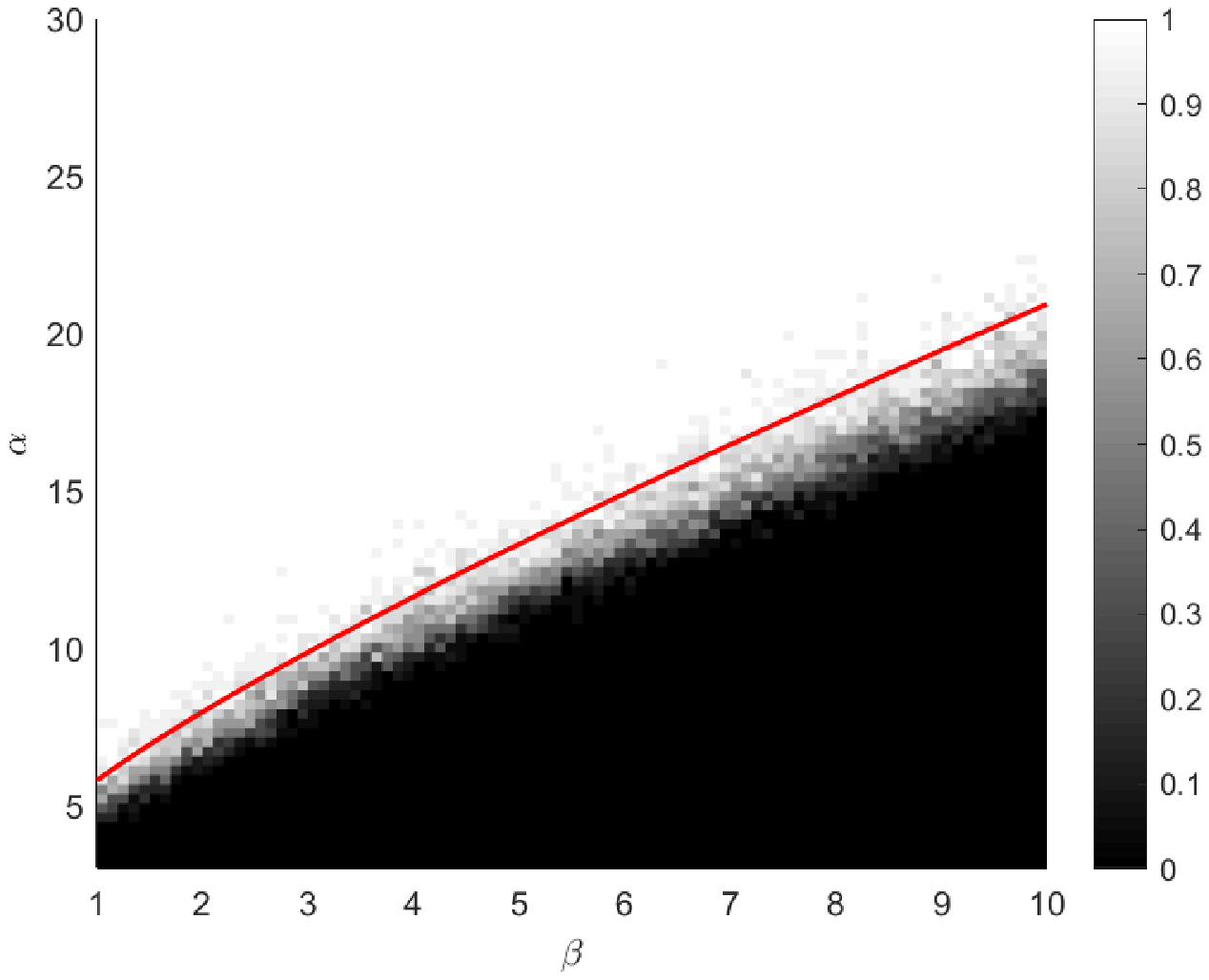}
		\caption{Normalized spectral clustering}
	\end{subfigure}%

	\caption{Empirical success rate of exact recovery for both spectral clustering methods. We fix $n=600$ and the number of trials to be 20. For each pair of $\al$ and $\be$, we run both methods and count how many times each method succeeds. Dividing by the number of trials, we obtain the empirical probability of success. The red line indicates the theoretical threshold $\sqrt{\al}-\sqrt{\be}=\sqrt{2}$ for strong consistency.} 
	\label{phase}
\end{figure}

\begin{figure}[h!]
	
	\begin{subfigure}[h]{0.5\linewidth}
		\includegraphics[width=\linewidth]{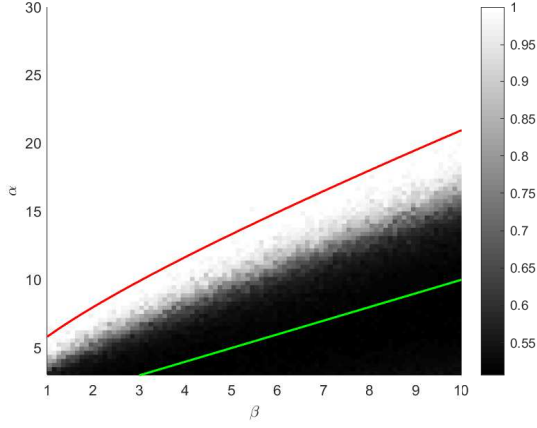}
		\caption{Unnormalized spectral clustering}
	\end{subfigure}
	\hfill
	\begin{subfigure}[h]{0.5\linewidth}
		\includegraphics[width=\linewidth]{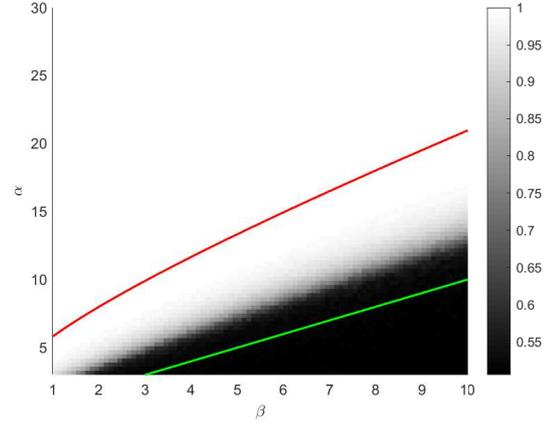}
		\caption{Normalized spectral clustering}
	\end{subfigure}%
	
	\caption{Empirical expectation of agreement for both spectral clustering methods. We fix $n=600$ and the number of trials to be 20. For each pair of $\al$ and $\be$ and each trial, we run both methods and calculate their agreements. Averaging over all trials, we obtain the empirical expectation of agreement. The red line indicates the theoretical threshold $\sqrt{\al}-\sqrt{\be}=\sqrt{2}$ for strong consistency. The green line is $\al=\be$, which serves as the theoretical boundary for weak consistency in this sparsity regime.} 
	\label{agreement}
\end{figure}

We further explore other possible choices of approximation $\tilde{u}_2$ to the second eigenvector of $L$ or $(L,D)$. Figure~\ref{approx} shows $\sqrt{n}\norminf{u_2-\tilde{u}_2}$ for different choices of $\tilde{u}_2$. These approximations can be interpreted from an iterative perspective. For example, our choice of $\tilde{u}_2=(D-\lambda_2(L)I)^{-1}A\utwostar$ for the unnormalized spectral clustering can be seen as the output of one-step fixed point iteration for solving the system $(D-\lambda_2(L)I)u=Au$ with initial guess $\utwostar$. The vector $\tilde{u}_2=(1-\lambda_2(\Lsym))^{-1}\Dneg A\utwostar=\Dneg A\utwostar/\lambda_2(\Dneg A)$ for the normalized spectral clustering can be seen as the output of an one-step power iteration on the matrix $\Dneg A$ with initial guess $\utwostar$, which is similar to the original idea in the paper of Abbe et al.\cite{AFWZ17}. We attempt to adopt the power iteration idea on the shifted Laplacian $\frac{n(p+q)}{2}P-L$, where $P=I-\frac{1}{n}J_{n\times n}$ is the projection onto the orthogonal complement space of span$\lc\ones_n\rc$. The purpose of introducing this shift is to make the Fiedler eigenvector correspond to the leading eigenvalue, and thus we can apply the idea of power iteration. However this idea does not seem to produce a satisfactory result. We also point out that the $\lambda_2(L)$ and $\lambda_2(\Lsym)$ in our approximations can be replaced with $\lambda_2(L^*)$ and $\lambda_2(\Lsym^*)$ respectively. Doing so will only introduce a higher order error in our analysis, which is confirmed by the results in Figure~\ref{approx}.
\begin{figure}[h!]
	
	\begin{subfigure}[h]{0.5\linewidth}
		\includegraphics[width=\linewidth]{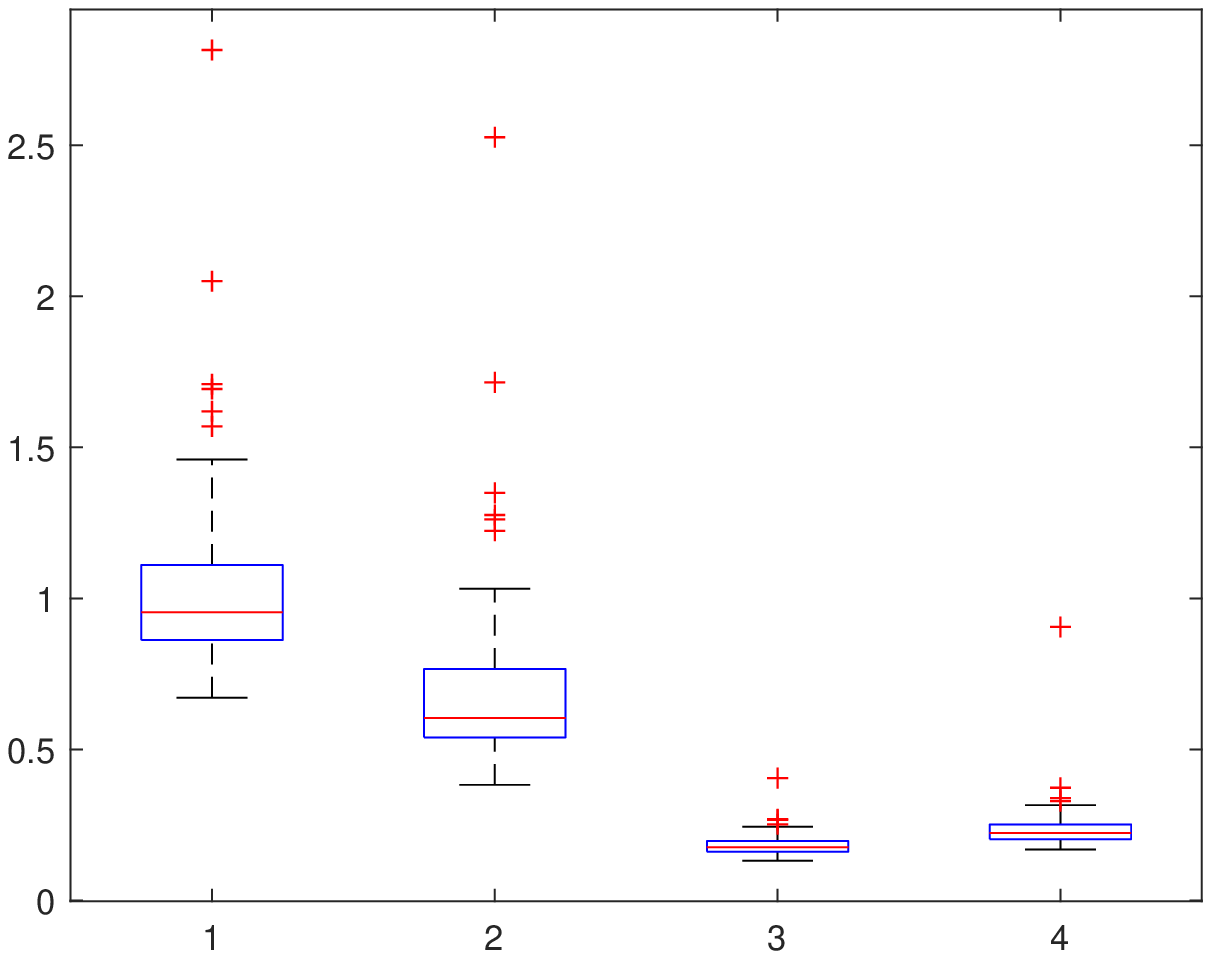}
		\caption{Unnormalized spectral clustering}
	\end{subfigure}
	\hfill
	\begin{subfigure}[h]{0.5\linewidth}
		\includegraphics[width=\linewidth]{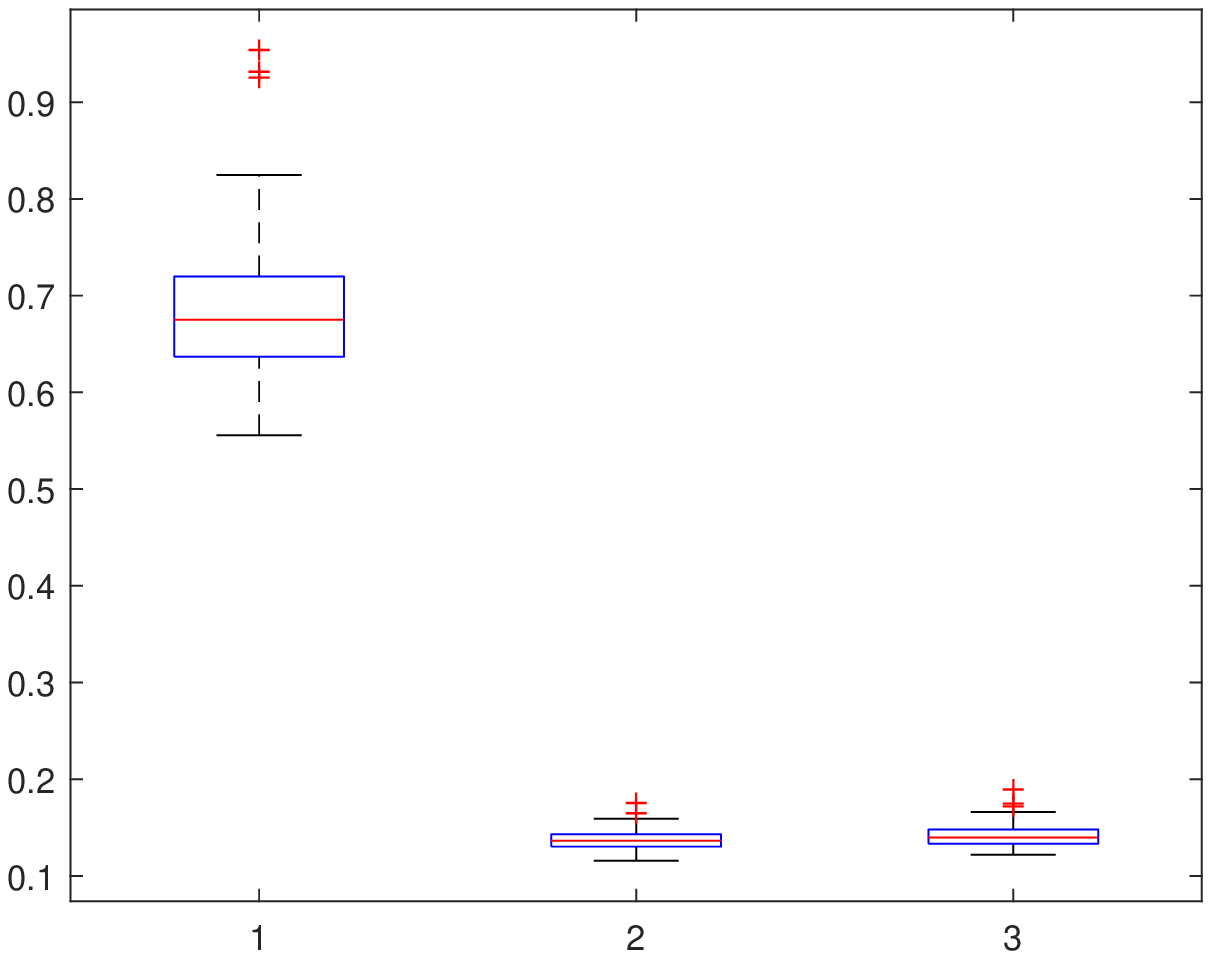}
		\caption{Normalized spectral clustering}
	\end{subfigure}%
	
	\caption{Boxplots showing $\sqrt{n}\norminf{u_2-\tilde{u}_2}$ (up to sign of $u_2$) for different choices of $\tilde{u}_2$. We fix $n=5000$, $\al=10$, $\be=2$ and the number of trials as 100. \textbf{Left: }(1) $\tilde{u}_2=\utwostar$, (2) $\tilde{u}_2= (\frac{n(p+q)}{2}P-L)\utwostar/(\frac{n(p+q)}{2}-\lambda_2(L))$ where $P=I-\frac{1}{n}J_{n\times n}$, (3) $\tilde{u}_2=(D-\lambda_2(L)I)^{-1}A\utwostar$, (4) $\tilde{u}_2=(D-\lambda_2(L^*)I)^{-1}A\utwostar$. \textbf{Right: }(1) $\tilde{u}_2=\utwostar$, (2) $\tilde{u}_2=(1-\lambda_2(\Lsym))^{-1}\Dneg A\utwostar$, (3) $\tilde{u}_2=(1-\lambda_2(\Lsym^*))^{-1}\Dneg A\utwostar$.}
	\label{approx}
\end{figure}

\section{Proofs}\label{s:proof}
\subsection{Proofs for Section 2.3}
\begin{proof}[{\bf Proof of Theorem~\ref{thm:DK}}]
	Let $N=V\Sigma V^H$ be the spectral decomposition of $N$. Define $\Nhalf=\Sigma^{\frac{1}{2}}V^H$ and $\Nneghalf=V\Sigma^{-\frac{1}{2}}$. Then $\lp\Nneghalf\rp^H M\Nneghalf$ is Hermitian and admits the spectral decomposition 
	\begin{equation}\label{eq:spectraldecomp}
	\lp\Nneghalf\rp^H M\Nneghalf=U\Lambda U^H
	\end{equation} 
	where $U$ is a unitary matrix and $\Lambda$ is a real diagonal matrix consisting of the eigenvalues. Left multiplying by $\Nneghalf$ and right multiplying by $\Nhalf$ on both sides in equation~\eqref{eq:spectraldecomp} gives
	$$N^{-1}M=X\Lambda X^{-1},$$
	where $X=\Nneghalf U$. We write
	\[
	r = (N^{-1}M-\hat{\lambda}I)\hat{u}=X\begin{bmatrix}
	\Lambda_1-\hat{\lambda}I&\\
	&\Lambda_2-\hat{\lambda}I
	\end{bmatrix} X^{-1}\hat{u}=X\begin{bmatrix}
	\Lambda_1-\hat{\lambda}I&\\
	&\Lambda_2-\hat{\lambda}I
	\end{bmatrix}\begin{bmatrix}
	\hat{c}\\
	\hat{s}
	\end{bmatrix},
	\]
	where $\hat{c}=Y_1^H\hat{u}$ and $\hat{s}=Y_2^H\hat{u}$. Multiplying both sides from the left by $\lp\Lambda_2-\hat{\lambda}I\rp^{-1}Y_2^H$ gives
	$$\hat{s}=\lp\Lambda_2-\hat{\lambda}I\rp^{-1}Y_2^Hr.$$
	Then 
	$$P\hat{u}=(Y_2^+)^HY_2^H\hat{u}=(Y_2^+)^H\hat{s}=(Y_2^+)^H\lp\Lambda_2-\hat{\lambda}I\rp^{-1}Y_2^Hr.$$
	Finally, note that
	\[
	\begin{bmatrix}
	X_1&X_1
	\end{bmatrix}^{-1}=U^H\Nhalf=\begin{bmatrix}
	U_1^H\\
	U_2^H
	\end{bmatrix}\Nhalf=\begin{bmatrix}
	Y_1^H\\
	Y_2^H
	\end{bmatrix},
	\]
	hence we  have $Y_2^H=U_2^H\Nhalf$. So
	\begin{align*}
	\norm{P\hat{u}}&\leq\norm{\Nneghalf}\norm{\lp U_2^H\rp^+}\norm{\lp\Lambda_2-\hat{\lambda}I\rp^{-1}}\norm{U_2^H}\norm{\Nhalf}\norm{r}\\
	&\leq\frac{\sqrt{\kappa(N)}\norm{(N^{-1}M-\hat{\lambda}I)\hat{u}}}{\delta}.
	\end{align*}

\end{proof}

\subsection{Proofs for Section 3.1}
\begin{proof}[{\bf Proof of Theorem~\ref{thm:Lsym-concen}}]
	We have
	\begin{align*}
	\norm{\Lsym-\Lsym^*}&=\norm{\Dneghalf A\Dneghalf-\Dstarneghalf A^*\Dstarneghalf}\\
	&\leq\norm{\Dneghalf(A-A^*)\Dneghalf}+\norm{\Dneghalf A^*\Dneghalf-\Dstarneghalf A^*\Dstarneghalf}.
	\end{align*}
	The first term on the right hand side is easily bounded by using Lemma~\ref{lem:A-concen} 
	$$\norm{\Dneghalf(A-A^*)\Dneghalf}\leq\frac{\norm{A-A^*}}{\dmin}\leq C_1(c_0,r)\frac{\sqrt{n\max_{ij}p_{ij}}}{\dmin}$$
	with probability at least $1-n^{-r}$.
	Denote $d=A\ones_n$ and $d^*=A^*\ones_n$, then the second term is bounded by
	\begin{align*}
	&\norm{\Dneghalf A^*\Dneghalf-\Dstarneghalf A^*\Dstarneghalf}\\
	\leq&\normfro{\Dneghalf A^*\Dneghalf-\Dstarneghalf A^*\Dstarneghalf}\\
	=&\sqrt{\sum_{i,j=1}^n p_{ij}^2\lp\frac{1}{\sqrt{d_id_j}}-\frac{1}{\sqrt{d^*_id^*_j}}\rp^2}\\
	\leq&\max_{ij}p_{ij}\sqrt{\sum_{i,j=1}^n\lp\frac{\sqrt{d_id_j}-\sqrt{d_i^*d_j^*}}{\sqrt{d_id_jd_i^*d_j^*}}\rp^2}\\
	\leq&\frac{\max_{ij}p_{ij}}{\dmin\dmin^*}\sqrt{\sum_{i,j=1}^n\lp\frac{d_id_j-d_i^*d_j^*}{\sqrt{d_id_j}+\sqrt{d_i^*d_j^*}}\rp^2}\\
	\leq&\frac{\max_{ij}p_{ij}}{2\min\{d_{\min},d^*_{\min}\}^3}\sqrt{\sum_{i,j=1}^n\lp d_id_j-d_i^*d_j^*\rp^2}\\
	=&\frac{\max_{ij}p_{ij}}{2\min\{d_{\min},d^*_{\min}\}^3}\normfro{\bd\bd^T-\bd^*\lp\bd^*\rp^T}\\
	\leq&\frac{\max_{ij}p_{ij}}{2\min\{d_{\min},d^*_{\min}\}^3}\lp\normfro{\bd(\bd-\bd^*)^T}+\normfro{(\bd-\bd^*)\lp\bd^*\rp^T}\rp\\
	=&\frac{\max_{ij}p_{ij}}{2\min\{d_{\min},d^*_{\min}\}^3}\lp\norm{\bd}\norm{\bd-\bd^*}+\norm{\bd-\bd^*}\norm{\bd^*}\rp\\
	\leq&\frac{\max_{ij}p_{ij}}{2\min\{d_{\min},d^*_{\min}\}^3}\lp\norm{A}\norm{A-A^*}\norm{\ones_n}^2+\norm{A^*}\norm{A-A^*}\norm{\ones_n}^2\rp\\
	\leq&\frac{\max_{ij}p_{ij}}{2\min\{d_{\min},d^*_{\min}\}^3}\lp\norm{A-A^*}+2\norm{A^*}\rp\norm{A-A^*}\norm{\ones_n}^2,
	\end{align*}
	where we have used the fact that $\normfro{uv^T}=\norm{uv^T}=\norm{u}\norm{v}$ for any $u$ and $v$. Again by using the bound for $(A-A^*)$, we get
	\begin{align*}
	&\norm{\Dneghalf A^*\Dneghalf-\Dstarneghalf A^*\Dstarneghalf}\\
	\leq&C_2(c_0,r)\frac{\max_{ij}p_{ij}}{\min\{d_{\min},d^*_{\min}\}^3}\lp \sqrt{n\max\nolimits_{ij}p_{ij}}+n\max\nolimits_{ij}p_{ij}\rp \sqrt{n\max\nolimits_{ij}p_{ij}}\cdot n\\
	\leq&C_3(c_0,r)\frac{\lp n\max_{ij}p_{ij}\rp^{5/2}}{\min\{d_{\min},d^*_{\min}\}^3}
	\end{align*} 
	with probability at least $1-n^{-r}$. Therefore combining the two terms we get
	\begin{align*}
	\norm{\Lsym-\Lsym^*}&\leq C_1(c_0,r)\frac{\sqrt{n\max_{ij}p_{ij}}}{\dmin}+C_3(c_0,r)\frac{\lp n\max_{ij}p_{ij}\rp^{5/2}}{\min\{d_{\min},d^*_{\min}\}^3}\\
	&=C_1(c_0,r)\frac{(\dmin^*)^2\sqrt{n\max_{ij}p_{ij}}}{\dmin(\dmin^*)^2}+C_3(c_0,r)\frac{\lp n\max_{ij}p_{ij}\rp^{5/2}}{\min\{d_{\min},d^*_{\min}\}^3}\\
	&\leq C_4(c_0,r)\frac{\lp n\max_{ij}p_{ij}\rp^{5/2}}{\min\{d_{\min},d^*_{\min}\}^3}
	\end{align*} 
	with probability at least $1-n^{-r}$.
\end{proof}

\subsection{Proofs for Section 3.2}
We start with some basic concentration inequalities.
\begin{lemma}\label{lem:concentration}
	\begin{enumerate}[(i)]
		\item (Chernoff) Let $\lc X_i\rc_{i=1}^n$ be independent variables. Assume $0\leq X_i\leq 1$ for each $i$. Let $X=X_1+\cdots+X_n$ and $\mu=\E{X}$. Then for any $t>0$,
		$$\mathbb{P}\left( |X-\mu|\geq t\right)\leq 2\exp\left( -\frac{t^2}{2\mu+t}\right). $$
		As a result, for any $r>0$, there exists $C=C(r)>0$ such that 
		$$\mathbb{P}\left( |X-\mu|\geq C\lp\log n+\sqrt{\mu\log n}\rp\right)\leq 2n^{-r}. $$
		\item (Bennett) Let $X\sim\text{Poisson}(\lambda)$. Then for any $0<x<\lambda$,
		$$\mathbb{P}\left( X\leq \lambda-x\right)\leq \exp\left( -\frac{x^2}{2\lambda}h\left(-\frac{x}{\lambda}\right)\right), $$
		where $h(u)=2u^{-2}((1+u)\log(1+u)-u)$.
		\item (Chebyshev) Let $X$ be a random variable with finite expected value $\mu$ and finite non-zero variance $\sigma^2$. Then for any real number $t > 0$,
		$$\prb{|X-\mu|\geq t}\leq\frac{\sigma^2}{t^2}.$$
		
	\end{enumerate}
\end{lemma}
\begin{proof}
	\begin{enumerate}[(i)]
		\item We omit the proof of the first inequality as it is a common form of the Chernoff bound. To prove the second inequality, we set $$\frac{t^2}{2\mu+t}=r\log n,$$
		which is $t=\frac{1}{2}\left( r\log n+\sqrt{r^2\log^2 n+8r\mu\log n}\right)\leq C(r)\lp\log n+\sqrt{\mu\log n}\rp$.
		\item The moment generating function of $X$ is
		$$\E{e^{\theta X}}=e^{\lambda(e^\theta-1)}$$
		for $\theta\in\mathbb{R}$. Fix $0<x<\lambda$, then for any $\theta>0$,
		\begin{align*}
		\prb{X\leq\lambda-x} & =\prb{e^{\theta X}\leq e^{\theta(\lambda-x)}}=\prb{e^{\theta(\lambda-x-X)}\geq 1 } \\
		& \leq e^{\theta(\lambda-x)}\E{e^{-\theta X}} =e^{\lp\lambda(e^{-\theta}-1)+\theta(\lambda-x)\rp}.
		\end{align*}
		The penultimate step is due to Markov's inequality. Finally, by setting $\theta=-\log\lp 1-\frac{x}{\lambda}\rp>0$ we get 
		$$\mathbb{P}\left( X\leq \lambda-x\right)\leq \exp\left( -\frac{x^2}{2\lambda}h\left(-\frac{x}{\lambda}\right)\right)$$
		as claimed.
	\end{enumerate}
\end{proof}
\subsubsection*{Unnormalized Laplacian}
\begin{proof}[{\bf Proof of Lemma~\ref{lem:A2toA1}}]
	$$\frac{\partial f}{\partial \xi}=-\log\left( 1-\frac{2\xi}{\al+\be}\right)>0 $$
	for $0<\xi<\frac{\al+\be}{2}$. So it suffices to prove $f(\frac{\al-\be}{2};\al,\be)>0$ when $\sqrt{\al}-\sqrt{\be}>\sqrt{2}$. Since $\al+\be>2\sqrt{\al\be}+2$, we have 
	\begin{align*}
	f\lp \frac{\al-\be}{2};\al,\be\rp =&\be\log\left( \frac{2\be}{\al+\be}\right)+\frac{\al-\be}{2}-1 \\
	>&\be\log\left( \frac{2\be}{\al+\be}\right)+\sqrt{\al\be}-\beta\\
	=&\be\left[ \sqrt{\frac{\al}{\be}}-\log\left( \frac{1}{2}+\frac{\al}{2\be}\right)-1 \right]. &
	\end{align*}
	It is straightforward to show by differentiation that $\sqrt{x}-\log\left(\frac{1}{2}+\frac{x}{2}\right)-1>0$ when $x>1$.
\end{proof}
The crucial step in controlling the minimum degree in the critical regime is the following Poisson approximation to binomials.
\begin{lemma}\label{lem:poisson}
	Let $X \sim \text{Binomial}(n/2, p)$ and $Y \sim \text{Binomial}(n/2, q)$ for $n$ even. Suppose $p=\al\log n/n$ and $q=\be\log n/n$ for constants $\al$ and $\be$. Let $\gamma=(\al+\be)/2$, then there exists $c_n\rightarrow 0$ depending on $\gamma$ such that for every $k\leq \gamma\log n$,
	\[
	\prb{X+Y=k}\leq\lp 1+c_n\rp n^{-\gamma}\frac{\lp \gamma\log n\rp^k}{k!}.
	\]
\end{lemma}

\begin{proof}
	For $k\leq \gamma\log n$,
	\begin{align*}
	\prb{X=k}&=\begin{bmatrix}
	n/2\\
	k
	\end{bmatrix}\lp\frac{\al\log n}{n}\rp^k\lp1-\frac{\al\log n}{n}\rp^{\frac{n}{2}-k}\\
	&=\frac{\frac{n}{2}\lp\frac{n}{2}-1\rp\cdots\lp\frac{n}{2}-k+1\rp}{k!}\cdot\frac{1}{\lp n/2\rp^k}\lp\frac{\al}{2}\log n\rp^k\lp1-\frac{\al\log n}{n}\rp^{\frac{n}{2}-k}\\
	&\leq \frac{1}{k!}\lp\frac{\al}{2}\log n\rp^k\lp1-\frac{\al\log n}{n}\rp^{\frac{n}{\al\log n}\lp\frac{\al}{2}\log n\rp\lp1-\frac{2k}{n}\rp}\\
	&\leq\lp 1+a_n\rp n^{-\frac{\al}{2}}\frac{\lp \frac{\al}{2}\log n\rp^k}{k!},
	\end{align*}
	where $a_n\rightarrow 0$ and is independent of $k$. The last inequality is due to
	$$\lim_{n\rightarrow\infty}\frac{\lp\frac{\al}{2}\log n\rp\lp1-\frac{2\gamma\log n}{n}\rp\log\lp1-\frac{\al\log n}{n}\rp^{\frac{n}{\al\log n}}}{-\frac{\al}{2}\log n}=1.$$
	Similarly there exists $b_n\rightarrow 0$ independent of $k$ such that\
	$$\prb{Y=k}\leq\lp 1+b_n\rp n^{-\frac{\be}{2}}\frac{\lp \frac{\be}{2}\log n\rp^k}{k!}.$$ 
	Finally note that
	\begin{align*}
	\prb{X+Y=k}&=\sum_{l=0}^{k}\prb{X=l}\prb{Y=k-l}\\
	&\leq (1+a_n)(1+b_n)n^{-\gamma}\frac{\lp \gamma\log n\rp^k}{k!}\\
	&:=\lp 1+c_n\rp n^{-\gamma}\frac{\lp \gamma\log n\rp^k}{k!}.
	\end{align*}
\end{proof}
With the help of the Poisson approximation we can now prove Lemma~\ref{lem:dmin}.
\begin{proof}[{\bf Proof of Lemma~\ref{lem:dmin}}]
	Let $d_i$ be the degree of the $i$th node. Let $X$ be a Poisson variable with mean $\frac{\al+\be}{2}\log n$. Then by Lemma~\ref{lem:concentration} and Lemma~\ref{lem:poisson}, for $n$ large enough
	\begin{align*}
	\prb{d_i\leq\frac{\al+\be}{2}\log n-\xi\log n}&\leq 2\prb{X\leq\frac{\al+\be}{2}\log n-\xi\log n}\\
	&\leq 2n^{-f(\xi;\al,\be)-1}.
	\end{align*}
	Taking union bound yields
	$$\prb{\dmin\geq \frac{\al+\be}{2}\log n-\xi\log n}\geq 1-2n^{-f(\xi;\al,\be)}.$$
\end{proof}

We prove Lemma~\ref{lem:dout} before we prove Theorem~\ref{thm:lambda(L)}.
\begin{proof}[\bf Proof of Lemma~\ref{lem:dout}]
	
	(i) Note that 
	$$\inner{\bdout-\bdout^*}{\ones_n}=2\sum_{i=1}^{\frac{n}{2}}\sum_{j=\frac{n}{2}+1}^{n}(A_{ij}-q).$$
	The result follows from the Chernoff bound.
	
	\noindent
	(ii)
	Let $\Aout$ denote the matrix after removing all edges within the same community in $A$. By Lemma~\ref{lem:A-concen},
	\begin{align*}
	\prb{\norm{\bdout-\bdout^*}\geq C(q_0,r)\sqrt{n^2q}}&=\prb{\norm{(\Aout-\Aout^*)\ones_n}\geq C(q_0,r)\sqrt{n^2q}}\\
	&\leq \prb{\norm{\Aout-\Aout^*}\geq C(q_0,r)\sqrt{nq}}\\
	&\leq n^{-r}.
	\end{align*}
	
	\noindent
	(iii)
	One can calculate the following two central moments of $X\sim\text{binomial}(n/2,q)$ by using the formula provided in~\cite{binom}:
	\begin{align*}
	&\E{\ls\lp X-\frac{nq}{2}\rp^2\rs}=\frac{1}{2}nq(1-q)\leq \frac{1}{2}nq\\
	&\text{var}\ls\lp X-\frac{nq}{2}\rp^2\rs=\frac{1}{2}nq(1-q)(nq - 6q - nq^2 + 6q^2 + 1)\leq \frac{1}{2}nq(nq+7).
	\end{align*}
	Let $X_i\stackrel{i.i.d}{\sim}\text{binomial}(n/2,q)$ and $Y_i\stackrel{i.i.d}{\sim}\text{binomial}(n/2,q)$. Then by letting $t=C_1(q_0)n^2q$ in Chebyshev's inequality,
	$$\prb{\sum_{i=1}^{n/2}\lp X_i-\frac{nq}{2}\rp^2\leq \lp\frac{1}{2}+C_1(q_0)\rp n^2q}\geq 1-\frac{\frac{1}{2}n^2q(nq+7)}{C_1(q_0)^2n^4q^2}\geq 1-\frac{1}{2}\lp\frac{1}{n}+\frac{0.01q_0}{n^2q}\rp.$$
	Same inequality holds for $Y_i$. By the union bound
	\begin{align*}
	&\prb{\norm{\bdout-\bdout^*}\leq C_2(q_0)\sqrt{n^2q}}\\
	=&\prb{\sqrt{\sum_{i=1}^{n/2}\lp X_i-\frac{nq}{2}\rp^2+\sum_{i=n/2+1}^{n}\lp Y_i-\frac{nq}{2}\rp^2}\leq C_2(q_0)\sqrt{n^2q}}\\
	\geq& 1-\lp\frac{1}{n}+\frac{0.01q_0}{n^2q}\rp.
	\end{align*}
	
\end{proof}
\begin{proof}[\bf Proof of Theorem~\ref{thm:lambda(L)}] 
	(i) Weyl's theorem shows
	\begin{align*}
	\lambda_3(L)&\geq \lambda_3(L^*)+\lambda_{\min}(D-D^*)-\norm{A-A^*}\\
	&=\frac{\al+\be}{2}\log n+\lp d_{\min}-\frac{\al+\be}{2}\log n\rp-\norm{A-A^*}\\
	&=d_{\min}-\norm{A-A^*}
	\end{align*}
	By Lemma~\ref{lem:dmin}, for $n$ large enough
	$$\prb{\dmin\geq \frac{\al+\be}{2}\log n-\xi\log n}\geq 1-2n^{-f(\xi;\al,\be)}.$$
	Then by Lemma~\ref{lem:A-concen},
	$$\prb{\norm{A-A^*}\leq C_1(\xi,\al,\be)\sqrt{\log n}}\geq 1-n^{-f(\xi;\al,\be)}.$$ 
	Therefore for $n\geq N=N(\xi,\al,\be,\epsilon)$,
	$$\prb{\lambda_3(L)\geq \frac{\al+\be}{2}\log n-(\xi+\epsilon)\log n}\geq 1-3n^{-f(\xi;\al,\be)}.$$
	Or equivalently for all $n$,
	$$\prb{\lambda_3(L)\geq \frac{\al+\be}{2}\log n-(\xi+\epsilon)\log n}\geq 1-C_2(\xi,\al,\be,\epsilon)n^{-f(\xi;\al,\be)}.$$
	(ii) By the min-max principle
	\begin{align*}
	\lambda_2(L)&=\min_{V\in\mathcal{V}_t}\max_{x\in V\backslash\{0\}}\frac{\inner{x}{Lx}}{\inner{x}{x}}\\
	&\leq \max_{x\in \text{span}\{\mathbbm{1_n},u_2^*\},||x||=1}\inner{x}{Lx}\\
	&=\inner{u_2^*}{Lu_2^*}\\
	&=\frac{2}{n}\inner{\bdout}{\ones_n}=nq+\frac{2}{n}\inner{\bdout-\bdout^*}{\ones_n}.
	\end{align*}
	The third step is due to $L\mathbbm{1}_n=0$ and $\ones_n\perp u_2^*$.
	
	\noindent
	(iii) Let $u_2$ be the eigenvector of $L$ that corresponds to $\lambda_2(L)$, We have
	\begin{align*}
	\lambda_2(L)=\inner{u_2}{Lu_2}&=\inner{(u_2-\utwostar)+\utwostar}{L((u_2-\utwostar)+\utwostar)}\\
	&=\inner{\utwostar}{L\utwostar}+2\inner{u_2-\utwostar}{L\utwostar}+\inner{u_2-\utwostar}{L(u_2-\utwostar)}\\
	&\geq\inner{\utwostar}{L\utwostar}+2\inner{u_2-\utwostar}{L\utwostar}\\
	&\geq nq+\frac{2}{n}\inner{\bdout-\bdout^*}{\ones_n}-2\norm{u_2-\utwostar}\norm{L\utwostar}\\
	&=nq+\frac{2}{n}\inner{\bdout-\bdout^*}{\ones_n}-\frac{4}{\sqrt{n}}\norm{u_2-\utwostar}\norm{\bdout}.
	\end{align*}
	Let $\theta$ be the angle between $u_2$ and $u_2^*$. Assume $\theta\in[0,\pi/2]$, because otherwise just let $u_2:=-u_2$. Then by letting $N=I$, $M=L$, $\hat{u}=u_2^*$,  $\hat{\lambda}=\lambda_2(L^*)$, $X_1=\begin{bmatrix}
	\frac{1}{\sqrt{n}}\ones_n & u_2
	\end{bmatrix}$ and $P$ be the projection matrix onto the orthogonal complement of $X_1$ in Theorem~\ref{thm:DK} we get
	$$\norm{P\utwostar}=\sin(\theta)\leq\frac{\norm{(L-L^*)\utwostar}}{\delta}=\frac{2\norm{\bdout-\bdout^*}}{\delta\sqrt{n}},$$
	where $\delta=\lambda_3(L)-\lambda_2(L^*)$ which we for now assume to be positive. Therefore
	\begin{equation}
	\norm{u_2-\utwostar}=\sqrt{2-2\cos(\theta)}\leq\sqrt{2}\sin(\theta)\leq\frac{2\sqrt{2}\norm{\bdout-\bdout^*}}{\delta\sqrt{n}}. \label{eq:u2-u2*}
	\end{equation}
	Thus
	\begin{equation}
	\lambda_2(L)\geq nq+\frac{2}{n}\inner{\bdout-\bdout^*}{\ones_n}-\frac{8\sqrt{2}}{\delta n}\norm{\bdout-\bdout^*}\norm{\bdout}. \label{eq:1}
	\end{equation}
	It remains to find a lower bound for $\delta$. If $p\geq p_0\log n/n$ then for any $r>0$, the Chernoff bound and Lemma~\ref{lem:A-concen} give
	$$\prb{\norm{D-D^*}\geq C_1(p_0,r)\sqrt{np\log n}}\leq 2n^{-r}$$
	and
	$$\prb{\norm{A-A^*}\geq C_2(p_0,r)\sqrt{np}}\leq n^{-r}.$$
	Therefore there exists $M(p_0,r)$ large enough, such that for $q$ satisfying
	$$n(p-q)\geq M\sqrt{np\log n},$$ we have
	\begin{align*}
	\delta=\lambda_3(L)-\lambda_2(L^*)&=\lp\lambda_3(L^*)-\lambda_2(L^*)\rp+\lp\lambda_3(L)-\lambda_3(L^*)\rp\\
	&\geq\frac{n(p-q)}{2}-\norm{D-D^*}-\norm{A-A^*}\\
	&\geq\frac{n(p-q)}{2\sqrt{2}}
	\end{align*}
	with probability at least $1-3n^{-r}$. Combining this and~\eqref{eq:1} concludes the first half of the statement. 
	
	If $p = \al\log n/n$, $q = \be\log n/n$ and $\alpha$ and $\beta$ satisfy (A1) so that there is some constant $0<\xi<(\al-\be)/2$ satisfying $f(\xi;\al,\be)>0$, then by part (i),
	\begin{equation}
	\prb{\lambda_3(L)\geq \beta\log n+\epsilon(\al,\be)\log n}\geq 1-C_3(\xi,\al,\be)n^{-f(\xi,\al,\be)}.\label{eq:lambda3L}
	\end{equation}
	Therefore
	\begin{equation}
	\prb{\delta\geq \epsilon(\al,\be)\log n}\geq 1-C_3(\xi,\al,\be)n^{-f(\xi,\al,\be)}.\label{eq:deltaL}
	\end{equation}
	and
	\begin{equation}
	\prb{\norm{u_2-\utwostar}\leq C_4(\al,\be,\xi)\frac{1}{\sqrt{\log n}}}\geq 1-C_5(\al,\be,\xi)n^{-\f}.\label{eq:L:u-u*}
	\end{equation}
	Finally combining~\eqref{eq:1},~\eqref{eq:deltaL} and Lemma~\ref{lem:dout} gives
	$$\prb{\lambda_2(L)\geq\be\log n-C_6(\al,\be,\xi)\sqrt{\log n}}\geq 1-C_7(\al,\be,\xi)n^{-f(\xi;\al,\be)}.$$
\end{proof}

\subsubsection*{Normalized Laplacian}
\begin{proof}[\bf Proof of Theorem~\ref{thm:lambda(Lsym)}]
	\begin{enumerate}[(i)]
		\item Let $u_2$ be the eigenvector of $(L,D)$ that corresponds to $\lambda_2(\Lsym)$. Using the min-max principle we get
		\begin{align*}
		\lambda_2(\Lsym)&=\min_{V\in\mathcal{V}_t}\max_{x\in V\backslash\{0\}}\frac{\inner{x}{\Lsym x}}{\inner{x}{x}}\\
		&\leq \max_{x\in \text{span}\lc\Dhalf\ones_n,\Dhalf \utwostar\rc}\frac{\inner{x}{\Lsym x}}{\inner{x}{x}} \qquad \\
		&=\frac{\left\langle \Dhalf \utwostar-x,\Lsym(\Dhalf \utwostar-x)\right\rangle }{||\Dhalf \utwostar-x||^2}\\
		&\leq\frac{\left\langle  \utwostar,L\utwostar\right\rangle }{\inner{\utwostar}{D\utwostar}-2||x||\sqrt{\inner{\utwostar}{D\utwostar}}+||x||^2}\\
		&\leq\frac{\left\langle  \utwostar,L\utwostar\right\rangle }{\inner{\utwostar}{D\utwostar}-2||x||\sqrt{\inner{\utwostar}{D\utwostar}}},
		\end{align*}
		where
		$$x=\frac{\inner{\utwostar}{D\ones_n}}{\inner{\ones_n}{D\ones_n}}\Dhalf\ones_n$$
		is the part of $\Dhalf \utwostar$ that is parallel to $\Dhalf \ones_n$. The third equality is because $\Dhalf\ones_n$ is in the null space of $\Lsym$. Therefore the Rayleigh quotient takes maximum in the direction orthogonal to $\Dhalf\ones_n$. The last inequality is valid because later we will see $\norm{x}\leq\frac{1}{2}\sqrt{\inner{\utwostar}{D\utwostar}}$. Next we aim to give an upper and lower bound for $\inner{\utwostar}{L\utwostar}$, an upper bound for $\left|\inner{\utwostar}{D\ones_n} \right| $ and a lower bound for $\inner{\utwostar}{D\utwostar}=\frac{1}{n}\inner{\ones_n}{D\ones_n}$. First by Lemma~\ref{lem:dout},
		\begin{equation}
		\inner{\utwostar}{L\utwostar}=nq+\frac{2}{n}\inner{\bdout-\bdout^*}{\ones_n}\leq nq+ C_1(q_0,r)\sqrt{q\log n}\label{eq:2}
		\end{equation} 
		with probability at least $1-n^{-r}$. By Chernoff,
		\begin{align*}
		\left| \inner{\utwostar}{D\utwostar}-\frac{n(p+q)}{2}\right| &=\left| \frac{1}{n}\inner{\ones_n}{D\ones_n}-\frac{n(p+q)}{2}\right|\\
		&=\left| \frac{1}{n}\sum_{i=1}^{n}d_i-\frac{n(p+q)}{2}\right|\\
		&=\left| \frac{1}{n}\lp\sum_{i=j}A_{ij}+2\sum_{i>j}A_{ij}\rp-\frac{n(p+q)}{2}\right|\\
		&\leq C_2(r)\lp\sqrt{\frac{p\log n}{n}}+\frac{\log n}{n}+\sqrt{p\log n}+\sqrt{q\log n}\rp\\
		&\leq C_3(r,p_0)\sqrt{p\log n}  \numberthis \label{eq:3}
		\end{align*}
		with probability at least $1-n^{-r}$. Finally by Chernoff, 
		\begin{align*}
		\left|\inner{\utwostar}{D\ones_n} \right|&=\frac{1}{\sqrt{n}}\left| \sum_{i=1}^{n/2}d_i-\sum_{i=n/2+1}^{n}d_i\right| \\
		&=\frac{1}{\sqrt{n}}\left| \sum_{i=1}^{n/2}\sum_{j=1}^{n/2}A_{ij}-\sum_{i=n/2+1}^{n}\sum_{j=n/2+1}^{n}A_{ij}\right|\\
		&\leq\frac{1}{\sqrt{n}}\lp\left| \sum_{i=1}^{n/2}\sum_{j=1}^{n/2}A_{ij}-\frac{n^2p}{4}\right| +\left| \sum_{i=n/2+1}^{n}\sum_{j=n/2+1}^{n}A_{ij}-\frac{n^2p}{4}\right|\rp\\
		&\leq C_4(r,p_0)\sqrt{np\log n}  \numberthis \label{eq:4}
		\end{align*}
		with probability at least $1-n^{-r}$.			
		Therefore by combining~\eqref{eq:3} and~\eqref{eq:4},
		$$\norm{x}=\frac{\left|\inner{\utwostar}{D\ones_n} \right|}{\sqrt{\inner{\ones_n}{D\ones_n}}}\leq\frac{C_4(r,p_0)\sqrt{p\log n}}{\sqrt{\frac{n(p+q)}{2}-C_3(r,p_0)\sqrt{p\log n}}}\leq C_5(r,p_0)\sqrt{\frac{\log n}{n}}$$
		for $N\geq N(r,p_0)$. This justifies the claim that $\norm{x}\leq\frac{1}{2}\sqrt{\inner{\utwostar}{D\utwostar}}$. Combining ~\eqref{eq:2},~\eqref{eq:3} and~\eqref{eq:4} yields
		\begin{align*}
		\lambda_2(\Lsym)&\leq\frac{\left\langle  \utwostar,L\utwostar\right\rangle }{\inner{\utwostar}{D\utwostar}-2||x||\sqrt{\inner{\utwostar}{D\utwostar}}}\\
		&\leq\frac{nq+ C_1(q_0,r)\sqrt{q\log n}}{\frac{n(p+q)}{2}-C_3(r,p_0)\sqrt{p\log n}-C_6(r,p_0)\sqrt{\frac{\log n}{n}}\cdot\sqrt{np}}\\
		&\leq\frac{2q}{(p+q)}+C_7(r,p_0,q_0)\frac{\sqrt{q\log n}}{np}
		\end{align*}
		with probability at least $1-3n^{-r}$ for $n>N(r,p_0)$. Or equivalently
		$$\prb{\lambda_2(\Lsym)\leq\frac{2q}{p+q}+C_7(r,p_0,q_0)\frac{\sqrt{q\log n}}{np}}\geq 1-C_8(r,p_0)n^{-r}$$
		for all $n$.
		
		\item By the Chernoff bound  and the union bound, for any $r>0$, there exists $p_0=p_0(r)$ large enough such that for $p\geq p_0\log n/n$,
		\begin{equation}
		\prb{\dmax\leq C_1(p_0,r)np}\geq 1-n^{-r}. \label{eq:5}
		\end{equation}
		and
		\begin{equation}
		\prb{\dmin\geq C_2(p_0,r)np}\geq 1-n^{-r}. \label{eq:6}
		\end{equation}
		We have
		\begin{align*}
		\lambda_2&=\frac{\inner{u_2}{Lu_2}}{\inner{u_2}{Du_2}}\\
		&=\frac{\inner{\utwostar}{L\utwostar}+2\inner{u_2-\utwostar}{L\utwostar}+\inner{u_2-\utwostar}{L(u_2-\utwostar)}}{\inner{\utwostar}{D\utwostar}+2\inner{u_2-\utwostar}{D\utwostar}+\inner{u_2-\utwostar}{D(u_2-\utwostar)}}\\
		&\geq\frac{\inner{\utwostar}{L\utwostar}-2||u_2-\utwostar||||L\utwostar||}{\inner{\utwostar}{D\utwostar}+2||u_2-\utwostar||||D\utwostar||+||u_2-\utwostar||^2||D||}\\
		&=\frac{\inner{\utwostar}{L\utwostar}-\frac{4}{\sqrt{n}}||u_2-\utwostar||\norm{\bdout}}{\inner{\utwostar}{D\utwostar}+\frac{2||u_2-\utwostar||\dmax}{\sqrt{n}}+||u_2-\utwostar||^2\dmax}.
		\end{align*}
		Combining~\eqref{eq:2},~\eqref{eq:3} and~\eqref{eq:5} gives
		$$\lambda_2(\Lsym)\geq\frac{nq+ C_3(q_0,r)\sqrt{q\log n}-\frac{4}{\sqrt{n}}||u_2-\utwostar||\norm{\bdout}}{\frac{n(p+q)}{2}+C_4(r,p_0)\sqrt{p\log n}+C_1(p_0,r)np\lp\frac{||u_2-\utwostar||}{\sqrt{n}}+||u_2-\utwostar||^2\rp}$$
		with probability at least $1-3n^{-r}$. It remains to find an upper bound for $\norm{u_2-\utwostar}$ through Davis-Kahan. In Theorem~\ref{thm:DK}, we let $M=L$, $N=D$, $\hat{\lambda}=\frac{2q}{p+q}$, $\hat{u}=\utwostar$, $X_1=\begin{bmatrix}
		\frac{1}{\sqrt{n}}\ones_n & u_2
		\end{bmatrix}$ and $P$ be the projection matrix onto the orthogonal complement of $X_1$. Since $\utwostar$ is orthogonal to $\ones_n$, we have $\norm{P\utwostar}=\sin(\theta)$ where $\theta\in[0,\pi/2]$ is the angle between $u_2$ and $\utwostar$. Therefore
		\begin{equation}
		\norm{u_2-\utwostar}=\sqrt{2-2\cos(\theta)}\leq\sqrt{2}\sin(\theta)\leq\frac{\sqrt{2}\norm{\lp\Dneg L-\hat{\lambda}I\rp\utwostar}}{\delta},\label{eq:7}
		\end{equation}
		where $\delta=\lambda_3(\Lsym)-\lambda_2(\Lsym^*)\geq\lambda_3(\Lsym^*)-\lambda_2(\Lsym^*)-\norm{\Lsym-\Lsym^*}=\frac{p-q}{p+q}-\norm{\Lsym-\Lsym^*}.$ Using Theorem~\ref{thm:Lsym-concen} in conjunction with~\eqref{eq:6} we get
		$$\prb{\norm{\Lsym-\Lsym^*}\leq \frac{C_4(p_0,r)}{\sqrt{np}}}\geq1-n^{-r}.$$
		Therefore there exists $M(p_0,r)>0$ such that
		$$\frac{p-q}{\sqrt{p}}\geq\frac{M}{\sqrt{n}}$$
		implies 
		$$\prb{\delta\geq\frac{p-q}{4p}}\geq 1-C_5(p_0,r)n^{-r}.$$
		To control the numerator in~\eqref{eq:7}, note that
		\begin{align*}
		||(\Dneg L-\hat{\lambda})\utwostar||&=2\sqrt{\frac{1}{n}\sum_{i=1}^{n}\left( \frac{\douti}{d_i}-\frac{nq}{n(p+q)}\right)^2 }\\
		&\leq\frac{2}{n(p+q)\dmin\sqrt{n}}\sqrt{\sum_{i=1}^{n}\left( np\douti-nq\dini\right)^2 }\\
		&\leq\frac{2}{np\dmin\sqrt{n}}\sqrt{\sum_{i=1}^{n}\left( np\lp \douti-\frac{nq}{2}\rp-nq\lp \dini-\frac{np}{2}\rp\right)^2 } \\
		&=\frac{2}{np\dmin\sqrt{n}}\norm{np\lp\bdout-\bdout^*\rp-nq\lp\bdin-\bdin^*\rp}\\
		&\leq\frac{2}{\dmin\sqrt{n}}\lp\norm{\bdout-\bdout^*}+\norm{\bdin-\bdin^*}\rp\\
		&=\frac{2}{\dmin\sqrt{n}}\lp\norm{(\Aout-\Aout^*)\ones_n}+\norm{(\Ain-\Ain^*)\ones_n}\rp\\
		&\leq\frac{2}{\dmin}\lp\norm{\Aout-\Aout^*}+\norm{\Ain-\Ain^*}\rp,
		\end{align*}
		where the second line follows from $\Ain=A-\Aout$ and $\bdin=\Ain\ones_n$.  Combining Lemma~\ref{lem:A-concen} and~\eqref{eq:6} we get
		$$\prb{||(\Dneg L-\hat{\lambda})\utwostar||\leq C_6(p_0,r)\frac{1}{\sqrt{np}}}\geq 1-2n^{-r}.$$
		Therefore 
		$$\prb{\norm{u_2-\utwostar}\leq C_7(p_0,r)\frac{\sqrt{np}}{n(p-q)}}\geq 1-C_8(p_0,r)n^{-r}.$$
		Finally,
		\begin{align*}
		\lambda_2(\Lsym)&\geq\frac{nq+ C_3(q_0,r)\sqrt{q\log n}-\frac{4}{\sqrt{n}}||u_2-\utwostar||\norm{\bdout}}{\frac{n(p+q)}{2}+C_4(r,p_0)\sqrt{p\log n}+C_1(p_0,r)np\lp\frac{||u_2-\utwostar||}{\sqrt{n}}+||u_2-\utwostar||^2\rp}\\
		&\geq \frac{nq+ C_3(q_0,r)\sqrt{q\log n}-4C_7(p_0,r)\frac{\sqrt{p}}{n(p-q)}\norm{\bdout}}{\frac{n(p+q)}{2}+C_4(r,p_0)\sqrt{p\log n}+C_1(p_0,r)np\lp C_7(p_0,r)\frac{\sqrt{p}}{n(p-q)}+C_7(p_0,r)^2\frac{np}{n^2(p-q)^2}\rp}\\
		&\geq \frac{2q}{p+q}-C_8(p_0,q_0,r)\lp\frac{\frac{q}{p}\sqrt{p\log n}+\frac{q\sqrt{np}}{p-q}+\sqrt{q\log n}+\frac{\sqrt{p}\norm{\bdout}}{n(p-q)}}{np}\rp\\
		&\geq \frac{2q}{p+q}-C_9(p_0,q_0,r)\lp\frac{\sqrt{q\log n}}{np}+\frac{nq+\frac{1}{\sqrt{n}}\norm{\bdout}}{n(p-q)\sqrt{np}}\rp\\		
		\end{align*}
		with probability at least $1-C_{10}(p_0,r)n^{-r}$.
		
		Now suppose $p=\al\log n/n$ and $q=\be\log n/n$ with $\al>2$. It is easy to see that there exists $\xi(\al,\be)\leq\frac{\al+\be}{2}$ such that $f(\xi;\al,\be)>0$. Then by Lemma~\ref{lem:dmin},
		\begin{equation}
		\prb{\dmin\geq C_{11}(\al,\xi)np}\geq 1-n^{-f(\xi;\al,\be)}. \label{eq:8}
		\end{equation}
		In this case the proof above still holds but  with $r=f(\xi;\al,\be)$. Therefore
		\begin{equation}
		\prb{\norm{u_2-\utwostar}\leq C_{12}(\al,\be,\xi)\frac{1}{\sqrt{\log n}}}\geq 1-C_{13}(\al,\be,\xi)n^{-f(\xi;\al,\be)} \label{eq:Lsym:u-u*}
		\end{equation}
		and
		$$\prb{\lambda_2(\Lsym)\geq \frac{2\be}{\al+\be}-C_{14}(\al,\be,\xi)\frac{1}{\sqrt{\log n}}}\geq 1-C_{15}(\al,\be,\xi)n^{-f(\xi;\al,\be)},$$
		where we have used Lemma~\ref{lem:dout} to bound $\norm{\bdout}$.
	\end{enumerate}
\end{proof}

\subsection{Proofs for Section 3.3}
Any statement involving eigenvectors are up to sign, meaning that for any eigenvector $u$, either $u$ or $-u$ will suit the statement. For example,
the expression  $\|u - v\|$ should be understood as  $\min_{s \in \{\pm 1\}} \|s u- v\|$.

\subsubsection*{Unnormalized spectral clustering}
Let $\Am$ be the matrix that $\Am_{ij}=A_{ij}$ when neither $i$ nor $j$ equals $m$ and $\Am_{ij}=A^*_{ij}$ when $i$ or $j$ equals $m$. Let $\Lm$ be the corresponding unnormalized Laplacian matrix of $\Am$. Let $u_2$ be the eigenvector of $L$ that corresponds to the second smallest eigenvalue $\lambda_2(L)$. Let $\um$ be the eigenvector of $\Lm$ that corresponds to the second smallest eigenvalue $\lambda_2(\Lm)$. The lemma below bounds $\norm{u_2-\um}$.
\begin{lemma}{\label{lem:L:u2-um}}
	There exists $\xi=\xi(\al,\be)>0$, $C_1,C_2>0$ depending on $\al$, $\be$ and $\xi$, such that $\f>0$ and
	$$\prb{\max_{1\leq m\leq n}\norm{u_2-\um}\leq C_1\norminf{u_2}}\geq 1-C_2n^{-\f}.$$
\end{lemma}
\begin{proof}
	In Theorem~\ref{thm:DK} we let $M=\Lm$, $N=I$, $\hat{u}=u_2$, $\hat{\lambda}=\lambda_2(L)$, $X_1=\begin{bmatrix}
	\frac{1}{\sqrt{n}}\ones_n&u_2
	\end{bmatrix}$. Then up to sign of eigenvectors, 
	\begin{equation}
	\norm{u_2-\um}\leq\frac{\sqrt{2}\norm{(\Lm-L)u_2}}{\delta_m},\label{eq:9}
	\end{equation}
	where $\delta_m=\lambda_3(\Lm)-\lambda_2(L)$. We first use Weyl's theorem to bound $\lambda_3(\Lm)$ from below. The proof is similar to Theorem~\ref{thm:lambda(L)} (i). We note that by the construction of $\Am$, the $(m,m)$-entry of $(\Dm-D^*)$ is 0 and the $(i,i)$-entry ($i\neq m$) only differ from $(d_i-d_i^*)$ by at most 1. Thus by Lemma~\ref{lem:dmin}, Lemma~\ref{lem:A-concen}, Lemma~\ref{lem:A2toA1} and the union bound, there exists $\xi(\al,\be)\leq\frac{\al-\be}{2}$ such that $f(\xi;\al,\be)>0$ and
	\begin{align*}
	\min_{1\leq m\leq n}\lambda_3(\Lm)&\geq \lambda_3(L^*)+\min_{1\leq m\leq n}\lc\lambda_{\min}(\Dm-D^*)-\norm{\Am-A^*}\rc\\
	&\geq\lambda_3(L^*)+\min\lc\lambda_{\min}(D-D^*)-1,0\rc-\max_{1\leq m\leq n}\norm{\Am-A^*}\\
	&=\min\lc d_{\min}-1,\frac{(\al+\be)\log n}{2}\rc-\max_{1\leq m\leq n}\norm{\Am-A^*}\\
	&\geq \be\log n+\epsilon_1(\al,\be,\xi)\log n
	\end{align*}
	with probability at least $1-C_1(\al,\be,\xi)n^{-f(\xi;\al,\be)}$. ($\Am$ does not strictly fit the setting of Lemma~\ref{lem:A-concen}. But note that the $m$th row and column of $\Am-A^*$ cancel to 0. Thus we are essentially applying Lemma~\ref{lem:A-concen} to a submatrix of $\Am-A^*$.) Using this in conjunction with Theorem~\ref{thm:lambda(L)} (ii), we have
	$$\prb{\min_{1\leq m\leq n}\delta_m\geq\epsilon_2(\al,\be,\xi)\log n}\geq 1-C_2(\al,\be,\xi)n^{-f(\xi;\al,\be)}.$$
	To bound the numerator in~\eqref{eq:9}, we consider bounding the $m$th entry of $(\Lm-L)u_2$ and the other entries separately. Let $v=(\Lm-L)u_2$ then
	\begin{equation}
	|v_m|=|\mrow{(\Lm-L)}u_2|=|\mrow{(L^*-L)}u_2|\leq\norminf{L^*-L}\norminf{u_2}.\label{eq:4.4.3.1}
	\end{equation}
	For $i\neq m$,
	\begin{align*}
	\left( \sum_{i\neq m}v_i^2\right) ^{1/2}&=\left( \sum_{i\neq m}(A^*_{im}-A_{im})^2\lp u_2^{(m)}-u_2^{(i)}\rp^2\right) ^{1/2}\\
	&\leq 2\norminf{u_2}\left( \sum_{i\neq m}(A^*_{im}-A_{im})^2\right) ^{1/2}\\
	&\leq 2\norminf{u_2}\normtwotoinf{A^*-A}\\
	&\leq 2\norminf{u_2}\norm{A^*-A}.\numberthis\label{eq:4.4.3.2}
	\end{align*}
	Therefore by the Chernoff bound and Lemma~\ref{lem:A-concen},
	$$\max_{1\leq m\leq n}\norm{(\Lm-L)u_2}\leq\lp\norminf{L^*-L}+2\norm{A-A^*}\rp\norminf{u_2}\leq C_4(\al,\be,\xi)\log n\norminf{u_2}$$
	with probability at least $1-C_3(\al,\be,\xi)n^{-\f}$. This concludes the proof.
\end{proof}
The next lemma gives an entrywise bound of $A(u_2-\utwostar)$, which is the at the center of both unnormalized and normalized spectral clustering.
\begin{lemma}\label{lem:L:A(u-u*)}
	There exist $C_1,C_2>0$ depending on $\al$, $\be$ and $\xi$ such that
	$$\prb{\norminf{A(u_2-\utwostar)}\leq C_1\frac{\log n}{\sqrt{n}\log\log n}}\geq 1-C_2n^{-\f}.$$
\end{lemma}
\begin{proof}
	All the statements in this proof hold for a probability at least $1-Cn^{-\f}$ for some $C=C(\al,\be,\xi)>0$. Asymptotic notations hide constants that depend on $\al$, $\be$ and $\xi$.
	We claim
	\begin{align}
	\norminf{A(u_2-\utwostar)}&=\bigo{\frac{\norminf{u_2}\log n}{\log\log n}}\label{eq:4.4.0}\\
	\norminf{u_2}&=O\left( \frac{1}{\sqrt{n}}\right) .\label{eq:4.4.00}
	\end{align}
	We first prove~\eqref{eq:4.4.0}. Then we use~\eqref{eq:4.4.0} to prove~\eqref{eq:4.4.00}. Finally combining them concludes the proof. To start, note that
	\begin{align*}
	\norminf{A(u_2-\utwostar)}=&\max_{1\leq m\leq n}\abs{\mrow{A}\lp u_2-\utwostar\rp}\\
	\leq&\max_{1\leq m\leq n}\abs{\mrow{A}\lp u_2-\um\rp}+\max_{1\leq m\leq n}\abs{\mrow{A}\lp \um-\utwostar\rp}\\
	\leq&\max_{1\leq m\leq n}\normtwotoinf{A}\norm{u_2-\um}+\max_{1\leq m\leq n}\abs{\mrow{A^*}\lp\um-\utwostar\rp}\\
	&+\max_{1\leq m\leq n}\abs{\mrow{(A-A^*)}\lp\um-\utwostar\rp}. \numberthis\label{eq:4.4.1}
	\end{align*}
	For the first term on the right hand side we have
	$$\normtwotoinf{A}\leq\normtwotoinf{A^*}+\norm{A-A^*}=\bigo{\sqrt{\log n}}$$
	and
	$$\max_{1\leq m\leq n}\norm{u_2-\um}=\bigo{\norminf{u_2}}.$$
	Therefore, it holds that
	\begin{equation}
	\max_{1\leq m\leq n}\normtwotoinf{A}\norm{u_2-\um}= \bigo{\sqrt{\log n}\norminf{u_2}}. \label{eq:4.4.2}
	\end{equation}
	For the second term we have
	\begin{align*}
	\max_{1\leq m\leq n}\abs{\mrow{A^*}\lp\um-\utwostar\rp}&\leq\max_{1\leq m\leq n}\normtwotoinf{A^*}\norm{\um-\utwostar}\\
	&\leq\normtwotoinf{A^*}\lp\max_{1\leq m\leq n}\norm{u_2-\um}+\norm{u_2-\utwostar}\rp\\
	&= \frac{\log n}{\sqrt{n}}\cdot\bigo{\norminf{u_2}+\frac{1}{\sqrt{\log n}}}, \numberthis\label{eq:4.4.3}
	\end{align*}
	where we have used~\eqref{eq:L:u-u*}. For the third term we can use the fact that the $m$th row of $A$ and $\um-\utwostar$ are independent, therefore by the row concentration property of $A$ (Lemma~\ref{lem:row-con}) and union bound, we have (by letting $a=\frac{\f+1}{\al}$ and $p=\al\log n/n$ in Lemma~\ref{lem:row-con})
	$$\max_{1\leq m\leq n}\abs{\mrow{(A-A^*)}\lp\um-\utwostar\rp}= \bigo{\max_{1\leq m\leq n}\norminf{w}\varphi\left( \frac{\norm{w}}{\sqrt{n}\norminf{w}}\right) \log n} $$
	where $w=\um-\utwostar$ and $\varphi(t)=(1\vee\log(1/t))^{-1}$ for $t>0$. $\varphi(x)$ is non-decreasing, $\varphi(t)/t$ is non-increasing and $\lim_{t\rightarrow 0}\varphi(t)=0$. For brevity we set $x=\sqrt{n}\norminf{w}$, $y = ||w||$, $\gamma=1/\sqrt{\log n}$ and
	$$(*)=\norminf{w}\varphi\left( \frac{\norm{w}}{\sqrt{n}\norminf{w}}\right) \log n.$$
	When $y/x\geq \gamma$ we have
	$$(*)=\frac{\log n}{\sqrt{n}}\cdot  y\cdot\frac{x}{y}\varphi\left(\frac{y}{x}\right)\leq\frac{\log n}{\sqrt{n}}\cdot\frac{y}{\gamma}\varphi(\gamma).$$
	When $y/x\leq \gamma$ we have
	$$(*)=\frac{\log n}{\sqrt{n}}\cdot  x\varphi\left(\frac{y}{x}\right)\leq\frac{\log n}{\sqrt{n}}\cdot  x\varphi(\gamma).$$
	Thus for any $x,y>0$ we always have
	$$(*)\leq\frac{\log n}{\sqrt{n}}\cdot\lp x\varphi(\gamma)+\frac{y}{\gamma}\varphi(\gamma)\rp$$
	Lemma~\ref{lem:L:u2-um} and~\eqref{eq:L:u-u*} give
	\begin{align*}
	\max_{1\leq m\leq n}x&=\sqrt{n}\max_{1\leq m\leq n}\norminf{\um-\utwostar}\\
	&\leq\sqrt{n}\left( \max_{1\leq m\leq n}\norm{\um-u_2}+\norminf{u_2}+\norminf{\utwostar}\right)\\
	&=\sqrt{n}\cdot O\left( \norminf{u_2}\right)
	\end{align*}
	and
	$$\max_{1\leq m\leq n}y=\max_{1\leq m\leq n}\norm{\um-\utwostar}\leq\max_{1\leq m\leq n}\norm{\um-u_2}+\norm{u_2-\utwostar}=O(\norminf{u_2}+\gamma).$$
	Therefore
	\begin{align*}
	\max_{1\leq m\leq n}\abs{\mrow{(A-A^*)}\lp\um-\utwostar\rp}&=\frac{\log n}{\sqrt{n}}\bigo{\max_{1\leq m\leq n}\lc x\varphi(\gamma)+\frac{y}{\gamma}\varphi(\gamma)\rc}\\
	&=\frac{\log n}{\sqrt{n}}\bigo{\sqrt{n}\norminf{u_2}\varphi(\gamma)+\frac{\norminf{u_2}}{\gamma}\varphi(\gamma)+\varphi(\gamma)}\\
	&=\frac{\log n}{\sqrt{n}}\bigo{\frac{\sqrt{n}}{\log \log n}\norminf{u_2}+\frac{\sqrt{\log n}}{\log \log n}\norminf{u_2}+\frac{1}{\log \log n}}\\
	&=\bigo{\frac{\log n}{\log\log n}\norminf{u_2}}\numberthis\label{eq:4.4.4}
	\end{align*}
Thus	\eqref{eq:4.4.0} follows after~\eqref{eq:4.4.1}-\eqref{eq:4.4.4}. To prove~\eqref{eq:4.4.00}, we expand
	\begin{align*}
	\norminf{u_2}&=\norminf{(D-\lambda_2(L) I)^{-1}Au_2}\\
	&\leq\norminf{(D-\lambda_2(L) I)^{-1}A\utwostar}+\norminf{(D-\lambda_2(L) I)^{-1}A(u_2-\utwostar)}.\numberthis\label{eq:4.4.5}
	\end{align*}
	Note that $\dmin\geq \be\log n+\Omega(\log n)$ and $\lambda_2(L)\leq \be\log n+\bigo{\log n/\sqrt{n}}$. It holds
	$$\norminf{(D-\lambda_2(L) I)^{-1}}\leq\frac{1}{\dmin-\lambda_2(L)}=\bigo{\frac{1}{\log n}}.$$
	Therefore the two terms on the right hand side of~\eqref{eq:4.4.5} are bounded by
	$$\norminf{(D-\lambda_2(L) I)^{-1}A\utwostar}=\bigo{\frac{1}{\log n}\norminf{A}\norminf{\utwostar}}=\bigo{\frac{1}{\sqrt{n}}},$$
	$$\norminf{(D-\lambda_2(L) I)^{-1}A(u_2-\utwostar)}=\bigo{\frac{1}{\log n}\norminf{A(u_2-\utwostar)}}=\bigo{\frac{1}{\log\log n}\norminf{u_2}}.$$
	Hence the second term of the right hand side of~\eqref{eq:4.4.5} is absorbed into the left hand side and~\eqref{eq:4.4.00} follows.
	
\end{proof}
\begin{proof}[\bf Proof of Theorem~\ref{thm:L:strong-consis}]
	For $i\leq n/2$, the $i$th entry of $A\utwostar$ can be written as 
	$$(A\utwostar)_i=\frac{1}{\sqrt{n}}\lp \sum_{j=1}^{n/2}A_{ij}-\sum_{j=n/2+1}^{n}A_{ij}\rp.$$
	Therefore by Lemma~\ref{lem:BinomialDiff}, there exists $\epsilon(\al,\be)>0$ such that
	$$\prb{(A\utwostar)_i\geq \epsilon\frac{\log n}{\sqrt{n}}}\geq 1- n^{-(\sqrt{\al}-\sqrt{\be})^2/2+\epsilon\log(\al/\be)/2}=1-o(n^{-1}).$$
	Similarly for $i\geq n/2+1$,
	$$\prb{(A\utwostar)_i\leq -\epsilon\frac{\log n}{\sqrt{n}}}=1-o(n^{-1}).$$
	Let $z_i=1$ if $i\leq n/2$ and $z_i=-1$ if $i\geq n/2+1$. By union bound
	\begin{equation}
	\prb{z_i\lp A\utwostar\rp_i \geq \eta_1(\al,\be)\frac{\log n}{\sqrt{n}}\;\text{ for all }i}=1-o(1).\label{eq:3.9.3}
	\end{equation}
	Using the fact that
	$$\prb{\dmax\leq C_1(\al)\log n}=1-o(1),$$ we get
	\begin{align}
	\prb{z_i\lp(D-\lambda_2(L)I)^{-1}A\utwostar\rp_i \geq \frac{\eta_2(\al,\be)}{\sqrt{n}}\;\text{ for all }i}=1-o(1).\label{eq:3.9.1}
	\end{align}
	Finally note that
	\begin{align}
	u_2=(D-\lambda_2(L)I)^{-1}A\utwostar+(D-\lambda_2(L)I)^{-1}A(u_2-\utwostar).\label{eq:3.9.2}
	\end{align}
	The proof is finished by combining~\eqref{eq:3.9.1},~\eqref{eq:3.9.2} and Lemma~\ref{lem:L:A(u-u*)}.
\end{proof}

\subsubsection*{Normalized spectral clustering}
Let $\Am$ be defined in the same way as we did in the unnormalized case. Let $u_2$ be the eigenvector of $(L,D)$ that corresponds to the second smallest eigenvalue $\lambda_2(\Lsym)$. Let $\um$ be the eigenvector of $(\Lm,\Dm)$ that corresponds to the second smallest eigenvalue $\lambda_2(\Lsymm)$. Readers should bear in mind the equivalence of the several eigenvalue problems regarding the normalized Laplacian (see Section 2.1).
\begin{lemma}\label{lem:Lsym:u2-u*}
	There exists $\xi=\xi(\al,\be)>0$, $C_1,C_2>0$ depending on $\al$, $\be$ and $\xi$, such that $\f>0$ and
	$$\prb{\max_{1\leq m\leq n}\norm{u_2-\um}\leq C_1\norminf{u_2}}\geq 1-C_2n^{-\f}.$$
\end{lemma}
\begin{proof}
	By Lemma~\ref{lem:dmin}, we can pick $\xi(\al,\be)\leq\frac{\al+\be}{2}$ such that $\f>0$ and
	\begin{equation}
	\prb{\dmin\geq C_1(\al,\be,\xi)\log n}\geq 1-C_2(\al,\be,\xi)n^{-\f}. \label{eq:4.6.1}
	\end{equation}
	Similar bound for maximum degree follows after the Chernoff bound.
	\begin{equation}
	\prb{\dmax\leq C_3(\al,\be,\xi)\log n}\geq 1-C_4(\al,\be,\xi)n^{-\f}. \label{eq:4.6.0}
	\end{equation}
	All the statements in the following proof hold for a probability at least $1-Cn^{-\f}$ for some $C=C(\al,\be,\xi)>0$ unless otherwise specified. Asymptotic notations hide constants that depend on $\al$, $\be$ and $\xi$. We first note that by construction of $\Am$,
	\begin{align*}
	\dminm & \geq \min\lc\dmin-1,\frac{\al+\be}{2}\log n\rc, \\
	\dmaxm & \leq \max\lc\dmax+1,\frac{\al+\be}{2}\log n\rc
	\end{align*}
	for all $m$. Therefore by~\eqref{eq:4.6.1} and~\eqref{eq:4.6.0} we have
	\begin{equation}
	\min_{1\leq m\leq n}\dminm=\Omega(\log n). \label{eq:4.6.10}
	\end{equation}
	and
	\begin{equation}
	\max_{1\leq m\leq n}\dmaxm=O(\log n). \label{eq:4.6.11}
	\end{equation}
	We decompose
	\begin{equation}
	u_2=a\frac{1}{\sqrt{n}}\ones_n+b\um+c\uo \label{eq:4.6.2}
	\end{equation}
	where $\uo$ is the unit vector that is orthogonal to span$\{\ones_n,\um\}$. Then
	\begin{align*}
	&	1=a^2+b^2+2ab\inner{\frac{1}{\sqrt{n}}\ones}{\um}+c^2, \\
	&\inner{u_2}{\um}=a\inner{\frac{1}{\sqrt{n}}\ones}{\um}+b.
	\end{align*}
	We aim to bound 
	\begin{align*}
	\norm{u_2-\um}&=\sqrt{2-2\inner{u_2}{\um}}\\
	& \leq \sqrt{2-2\inner{u_2}{\um}^2}\\
	&=\sqrt{2a^2\lp 1-\inner{\frac{1}{\sqrt{n}}\ones_n}{\um}^2\rp+2c^2}\\
	&\leq\sqrt{2}(|a|+|c|). \numberthis\label{eq:4.6.3}
	\end{align*}
	We will use the term $|c|$ to bound $|a|$ and Davis-Kahan to bound $|c|$. Taking inner product with $\frac{1}{\sqrt{n}}\Dm\ones_n$ on both sides of~\eqref{eq:4.6.2} yields
	\begin{equation}
	\inner{\frac{1}{\sqrt{n}}\ones_n}{\Dm u_2}=a\inner{\frac{1}{\sqrt{n}}\ones_n}{\frac{1}{\sqrt{n}}\Dm\ones_n}+c\inner{\frac{1}{\sqrt{n}}\ones_n}{\Dm\uo}, \label{eq:4.6.4}
	\end{equation}
	where we have used the fact that $\inner{\ones_n}{\Dm\um}=0$. Note that $\inner{\ones_n}{Du_2}=0$, we have
	\begin{align*}
	\max_{1\leq m\leq n}\left| \inner{\frac{1}{\sqrt{n}}\ones_n}{\Dm u_2}\right| &= \max_{1\leq m\leq n}\left| \inner{\frac{1}{\sqrt{n}}\ones_n}{(\Dm-D) u_2}\right| \\
	&\leq\frac{\norminf{u_2}}{\sqrt{n}} \max_{1\leq m\leq n}\sum_{i=1}^{n}\left|\dm_i-d_i\right| \\
	&=\bigo{\frac{\log n}{\sqrt{n}}\norminf{u_2}}, \numberthis\label{eq:4.6.5}
	\end{align*}
	where the last step is due to the Chernoff bound. Indeed, when $i\neq m$, by construction of $\Am$, $|\dm_i-d_i|=|A_{im}-A^*_{im}|\leq 1.$ And it is easy to see that $\E{|A_{im}-A^*_{im}|}\leq 2p$. Therefore the Chernoff bound gives
	$$\prb{\sum_{i\neq m}|\dm_i-d_i|=\bigo{\log n}}\geq 1-n^{-\f-1}.$$
	When $i=m$ we use the Chernoff bound again,
	$$\prb{|\dm_m-d_m|=|d_m-d_m^*|=\bigo{\log n}}\geq 1-n^{-\f-1}.$$
	Thus by the union bound we have $$\max_{1\leq m\leq n}\sum_{i=1}^{n}\left|\dm_i-d_i\right|=\bigo{\log n}$$
	which proves the last step of~\eqref{eq:4.6.5}. We proceed to use the almighty Chernoff and the union bound once again,
	\begin{align*}
	\min_{1\leq m\leq n}\inner{\frac{1}{\sqrt{n}}\ones_n}{\frac{1}{\sqrt{n}}\Dm\ones_n}&=\min_{1\leq m\leq n}\frac{1}{n}\left| \sum_{i=j}\Am_{ij}+2\sum_{i>j}\Am_{ij}\right|\\
	&\geq\frac{(\al+\be)\log n}{2}-\bigo{\frac{\log n}{\sqrt{n}}}\\
	&=\Omega(\log n). \numberthis\label{eq:4.6.6}
	\end{align*}
	Then by~\eqref{eq:4.6.11},
	\begin{equation}
	\max_{1\leq m\leq n}\inner{\frac{1}{\sqrt{n}}\ones_n}{\Dm\uo}\leq	\max_{1\leq m\leq n}\norm{\Dm}=\bigo{\log n}.\label{eq:4.6.7}	\end{equation}
	Combining~\eqref{eq:4.6.4}-\eqref{eq:4.6.7} we get
	\begin{equation}
	\max_{1\leq m\leq n}|a|=\bigo{\frac{1}{\sqrt{n}}\norminf{u}+\max_{1\leq m\leq n}|c|}. \label{eq:4.6.15}
	\end{equation}
	It remains to bound $|c|$ through Davis-Kahan. In Theorem~\ref{thm:DK} we let $M=\Am$, $N=\Dm$, $\hat{\lambda}=\lambda_2(A,D)=1-\lambda_2(\Lsym)$, $\hat{u}=u_2$, $X_1=\begin{bmatrix}
	\frac{1}{\sqrt{n}}\ones_n&\um
	\end{bmatrix}$. Then
	\begin{equation}
	|c|=|\sin\theta|\leq \frac{\sqrt{\kappa(\Dm)}\norm{\lp(\Dm)^{-1}\Am-\Dneg A\rp u_2}}{\delta_m} \label{eq:4.6.8}
	\end{equation}
	where $\theta\in[0,\pi/2]$ is the angle between $u_2$ and $\uo$, $$\delta_m=\lambda_2(A,D)-\lambda_3(\Am,\Dm)=\lambda_3(\Lsymm)-\lambda_2(\Lsym)\geq \lambda_3(\Lsym^*)-\lambda_2(\Lsym)-\norm{\Lsymm-\Lsym^*}.$$
	By applying~\eqref{eq:4.6.10} in Theorem~\ref{thm:Lsym-concen} we have 
	$$\max_{1\leq m\leq n}\norm{\Lsymm-\Lsym^*}=\bigo{\frac{1}{\sqrt{\log n}}}.$$
	(Although $\Lsymm$ does not strictly fit the setting of Theorem~\ref{thm:Lsym-concen}, readers can check that the bound above is true by referring to the proof of Theorem~\ref{thm:Lsym-concen}. Specifically all we need is $\max_{1\leq m\leq n}\norm{\Am-A^*}=\bigo{\sqrt{\log n}}$, which is guaranteed by Lemma~\ref{lem:A-concen}.) Thus combining this and Theorem~\ref{thm:lambda(Lsym)} (i) we have
	\begin{equation}
	\min_{1\leq m\leq n}\delta_m\geq\lambda_3(\Lsym^*)-\lambda_2(\Lsym)-\max_{1\leq m\leq n}\norm{\Lsymm-\Lsym^*}=\Omega(1).\label{eq:4.6.9}
	\end{equation}
	It follows immediately after~\eqref{eq:4.6.10} and~\eqref{eq:4.6.11} that
	\begin{equation}
	\max_{1\leq m\leq n}\kappa(\Dm)=O(1). \label{eq:4.6.12}
	\end{equation}
	Finally we need to bound the numerator in~\eqref{eq:4.6.8}. Let $v=((\Dm)^{-1}\Am-D^{-1}A)u_2$. We consider bounding the $m$th entry of $v$ and other entries separately. When $i\neq m$,
	$$|v_i|=\left| \left( \frac{A^*_{im}}{\dm_i}-\frac{A_{im}}{d_i}\right)(u_2)_m+\sum_{j\neq m}\left( \frac{1}{\dm_i}-\frac{1}{d_i}\right)A_{ij}(u_2)_j  \right|. $$
	Using the fact that $\dm_i-d_i=A^*_{im}-A_{im}$ and~\eqref{eq:4.6.11},~\eqref{eq:4.6.0} we can bound $|v_i|$ by
	\begin{align*}
	|v_i|&\leq\norminf{u_2}\cdot\lp \frac{\left|A^*_{im}\lp d_i-\dm_i\rp+\dm_i\lp A^*_{im}-A_{im}\rp\right|}{\dm_i d_i}+\sum_{j\neq m}  \frac{\left|d_i-\dm_i\right|}{\dm_i d_i}A_{ij} \rp\\
	&=\norminf{u_2}\cdot\lp \frac{\left|\lp\dm_i-A^*_{im}\rp\lp A^*_{im}-A_{im}\rp\right|}{\dm_i d_i}+\sum_{j\neq m}  \frac{\left|A^*_{im}-A_{im}\right|}{\dm_i d_i}A_{ij} \rp\\
	&\leq \norminf{u_2}\cdot\lp \frac{\left| A^*_{im}-A_{im}\right|}{d_i}+ \frac{\left|A^*_{im}-A_{im}\right|}{\dm_i } \rp\\
	&=\bigo{\frac{\norminf{u_2}\left| A_{im}-A^*_{im}\right|}{\log n}}.
	\end{align*}
	Therefore 
	\begin{align*}
	\left( \sum_{i\neq m}v_i^2\right)^{1/2}&=\bigo{\frac{\norminf{u_2}}{\log n}\left( \sum_{i\neq m}\left( A_{im}-A^*_{im}\right) ^2\right)^{1/2}}\\
	&=\bigo{\frac{\norminf{u_2}}{\log n}\normtwotoinf{A-A^*}}\\
	&=\bigo{\frac{\norminf{u_2}}{\log n}\norm{A-A^*}}=\bigo{\frac{\norminf{u_2}}{\sqrt{\log n}}}. 
	\end{align*}
	When $i=m$,
	\begin{align*}
	|v_m|=\abs{\sum_{j=1}^{n}\lp\frac{A^*_{mj}}{d^*_m}-\frac{A_{mj}}{d_m}\rp(u_2)_j}&\leq\norminf{u_2}\abs{\sum_{j=1}^{n}\lp\frac{A^*_{mj}}{d^*_m}-\frac{A_{mj}}{d_m}\rp}\\
	&\leq\norminf{u_2}\lp\sum_{j=1}^{n}\frac{A^*_{mj}}{d^*_m}+\sum_{j=1}^{n}\frac{A_{mj}}{d_m}\rp\\
	&=2\norminf{u_2}.
	\end{align*}
	Thus $\norm{v}=\bigo{\norminf{u_2}}$. Note that what we used to bound $\norm{v}$ are~\eqref{eq:4.6.11},~\eqref{eq:4.6.0} and $\norm{A-A^*}=\bigo{\sqrt{\log n}}$, which are independent of $m$. Hence
	\begin{equation}
	\max_{1\leq m\leq n}\norm{((\Dm)^{-1}\Am-D^{-1}A)u_2}=\bigo{\norminf{u_2}}.\label{eq:4.6.13}
	\end{equation}
	It follows after~\eqref{eq:4.6.8},~\eqref{eq:4.6.12},~\eqref{eq:4.6.7} and~\eqref{eq:4.6.13} that
	\begin{equation}
	\max_{1\leq m\leq n}|c|=\bigo{\norminf{u_2}}.\label{eq:4.6.14}
	\end{equation}
	The proof concludes after combining~\eqref{eq:4.6.3},~\eqref{eq:4.6.15} and~\eqref{eq:4.6.14}.
\end{proof}

\begin{lemma}\label{lem:Lsym:A(u-u*)}
	There exist $C_1,C_2>0$ depending on $\al$, $\be$ and $\xi$ such that
	$$\prb{\norminf{A(u_2-\utwostar)}\leq C_1\frac{\log n}{\sqrt{n}\log\log n}}\geq 1-C_2n^{-\f}.$$
\end{lemma}
\begin{proof}
	Similar to the proof of Lemma~\ref{lem:L:A(u-u*)}, we will prove the following two claims:
	\begin{align}
	\norminf{A(u_2-\utwostar)}&=\bigo{\frac{\norminf{u_2}\log n}{\log\log n}},\label{eq:4.7.0}\\
	\norminf{u_2}&=O\left( \frac{1}{\sqrt{n}}\right) .\label{eq:4.7.00}
	\end{align}
	For~\eqref{eq:4.7.0}, we refer to the proof of~\eqref{eq:4.4.0} in Lemma~\ref{lem:L:A(u-u*)}. Although $u_2$ in Lemma~\ref{lem:L:A(u-u*)} is  the second eigenvector of $L$ and here $u_2$ is the the second eigenvector of $(L,D)$, one can observe that all we need for the proof of~\eqref{eq:4.4.0} to hold are
	$$\max_{1\leq m\leq n}\norm{u_2-\um}=\bigo{\norminf{u_2}}$$
	and
	$$\norm{u_2-\utwostar}=\bigo{\frac{1}{\sqrt{\log n}}}.$$
	The former is guaranteed by Lemma~\ref{lem:Lsym:A(u-u*)} and the latter by~\eqref{eq:Lsym:u-u*}. Therefore we have proved~\eqref{eq:4.7.0}. To prove~\eqref{eq:4.7.00}, we expand
	\begin{align*}
	\norminf{u_2}&=\norminf{\frac{1}{1-\lambda_2(\Lsym)}\Dneg Au_2}\\
	&\leq\norminf{\frac{1}{1-\lambda_2(\Lsym)}\Dneg A\utwostar}+\norminf{\frac{1}{1-\lambda_2(\Lsym)}\Dneg A(u_2-\utwostar)}.\numberthis\label{eq:4.7.1}
	\end{align*}
	By Theorem~\ref{thm:lambda(Lsym)} and the bound for $\dmin$ we have
	$$\norminf{\frac{1}{1-\lambda_2(\Lsym)}\Dneg}=\bigo{\frac{1}{\log n}}.$$
	Therefore the two terms on the right hand side of~\eqref{eq:4.7.1} are bounded by
	$$\norminf{\frac{1}{1-\lambda_2(\Lsym)}\Dneg A\utwostar}=\bigo{\frac{1}{\log n}\norminf{A}\norminf{\utwostar}}=\bigo{\frac{1}{\sqrt{n}}},$$
	$$\norminf{\frac{1}{1-\lambda_2(\Lsym)}\Dneg A(u_2-\utwostar)}=\bigo{\frac{1}{\log n}\norminf{A(u_2-\utwostar)}}=\bigo{\frac{1}{\log\log n}\norminf{u_2}}.$$
	Hence the second term of right hand side of~\eqref{eq:4.7.1} is absorbed into the left hand side and~\eqref{eq:4.7.00} follows.
\end{proof}
\begin{proof}[\bf Proof of Theorem~\ref{thm:Lsym:strong-consis}]
	By\eqref{eq:3.9.3},
	\begin{align*}
	\prb{z_i\lp A\utwostar\rp_i \geq \eta_1(\al,\be)\frac{\log n}{\sqrt{n}}\;\text{ for all }i}=1-o(1).
	\end{align*}
	Using this in conjunction with Theorem~\ref{thm:lambda(Lsym)} (i), we have
	\begin{align}
	\prb{z_i\lp \frac{1}{1-\lambda_2(\Lsym)}\Dneg A\utwostar\rp_i \geq \eta_2(\al,\be)\frac{1}{\sqrt{n}}\;\text{ for all }i}=1-o(1).\label{eq:3.10.1}
	\end{align}
	Finally note that
	\begin{align}
	u_2=\frac{1}{1-\lambda_2(\Lsym)}\Dneg A\utwostar+\frac{1}{1-\lambda_2(\Lsym)}\Dneg A(u_2-\utwostar) \label{eq:3.10.2}
	\end{align}
	The proof is finished by combining~\eqref{eq:3.10.1},~\eqref{eq:3.10.2} and Lemma~\ref{lem:Lsym:A(u-u*)}.
\end{proof}

\end{document}